\documentclass[lettersize,journal]{IEEEtran}
\usepackage{cite}

\usepackage[pass]{geometry}
\usepackage{amsmath}
\usepackage{amssymb}
\usepackage{bm}
\usepackage{mathtools}
\usepackage{amsthm}
\usepackage{algorithm}
\usepackage{algorithmic}
\usepackage{tikz}
\usetikzlibrary{positioning}
\usetikzlibrary{arrows}
\usetikzlibrary{arrows.meta}
\usetikzlibrary{shapes.geometric}
\usetikzlibrary{matrix}

\usepackage{comment}
\usepackage{xcolor}
\usepackage{framed}

\newcommand{\SE}{\text{SE}(3)}
\newcommand{\SO}{\text{SO}(3)}
\newcommand{\R}{\mathbb{R}}

\newcommand{\se}{\mathfrak{se}(3)}

\newcommand{\cI}{\mathbb{I}}

\newcommand{\Ad}{\mathbf{Ad}}
\newcommand{\ad}{\mathbf{ad}}
\newcommand{\0}{\mathbf{0}}
\newcommand{\1}{\mathbf{1}}
\newcommand{\Add}{\mathfrak{Ad}}

\newcommand{\EX}{\mathfrak{X}}
\newcommand{\ex}{\mathfrak{x}}
\newcommand{\F}{\mathfrak{F}}
\newcommand{\N}{\mathcal{N}}
\newcommand{\M}{\mathcal{M}}
\newcommand{\diag}{\mathbf{diag}}

\newtheorem{theorem}{Theorem}
\newtheorem{remark}{Remark}
\newtheorem{lemma}{Lemma}
\newtheorem{problem}{Problem}
\newtheorem{proposition}{Proposition}
\newtheorem{corollary}{Corollary}
\newtheorem{definition}{Definition}

\begin{document}

\title{Fast and Modular Whole-Body Lagrangian Dynamics of Legged Robots with Changing Morphology}
%

\author{Sahand~Farghdani,~Omar Abdelrahman,~and~Robin~Chhabra~\IEEEmembership{Senior Member,~IEEE}
\thanks{S. Farghdani and O. Abdelrahman are with the Autonomous Space Robotics and Mechatronics Laboratory, Carleton University, Ottawa, ON, Canada\\
\indent R. Chhabra is with the Mechanical, Industrial, and Mechatronics Engineering Department, Toronto Metropolitan University, Toronto, ON, Canada}
}


\maketitle

\begin{abstract}

Fast and modular modeling of multi-legged robots (MLRs) is essential for resilient control, particularly under significant morphological changes caused by mechanical damage. Conventional fixed-structure models, often developed with simplifying assumptions for nominal gaits, lack the flexibility to adapt to such scenarios. To address this, we propose a fast modular whole-body modeling framework using Boltzmann-Hamel equations and screw theory, in which each leg’s dynamics is modeled independently and assembled based on the current robot morphology.
This singularity-free, closed-form formulation enables efficient design of model-based controllers and damage identification algorithms. Its modularity allows autonomous adaptation to various damage configurations without manual re-derivation or retraining of neural networks. We validate the proposed framework using a custom simulation engine that integrates contact dynamics, a gait generator, and local leg control. Comparative simulations against hardware tests on a hexapod robot with multiple leg damage confirm the model's accuracy and adaptability.
Additionally, runtime analyses reveal that the proposed model is approximately three times faster than real-time, making it suitable for real-time applications in damage identification and recovery.



\end{abstract}

\begin{IEEEkeywords}
Multi-legged Robot, Boltzmann-Hamel equations, Modular modeling, Singularity-free dynamics, Lie groups.
\end{IEEEkeywords}

\IEEEpeerreviewmaketitle

\section{Introduction}

\IEEEPARstart{T}{he} development of intelligent and adaptable walking robots has garnered increasing attention due to their potential across a wide range of applications, including anti-terrorism and military operations \cite{raibert2008bigdog}, space exploration, search and rescue missions \cite{ref44, arm2023scientific}, inspection tasks \cite{Spot, Anymal}, and delivery services \cite{gong2023legged}. Advances in hardware design have greatly enhanced the mobility of Multi-Legged Robots (MLRs), enabling their deployment in remote and hazardous environments. However, despite these improvements, most MLRs continue to demonstrate conservative autonomous behaviors due to several key challenges.


First, the complex dynamics of MLRs make it difficult to create fast, real-time models for motion prediction and control. Second, in cluttered or hostile environments, MLRs must plan routes and gaits efficiently while managing a high risk of collisions, which can lead to performance deterioration or unexpected physical damage. The simplified models traditionally used for guidance, navigation, and control may no longer apply in these cases. To overcome these challenges, there is a clear need for a whole-body dynamic model that is both fast and modular. Such a model should avoid simplifications while adapting easily to various damage scenarios.

Modularity, as defined in this paper, refers to the ability to add or remove identical legs at different locations on the main body and automatically derive the symbolic model without differentiation of expressions. Further, our definition of ``damage" is specifically limited to scenarios such as locked joints, partially broken components, or the complete loss of robot’s limbs/legs. Here, we do not address sensory damage, joint looseness, or computer hardware problems.

This paper presents a fast and modular approach to modeling MLRs that is capable of adapting to damage scenarios without the need for retraining neural networks or extensive manual re-derivation. Such an approach is essential for real-time planning, control, and state estimation of MLRs.



\subsection{Contributions}


When damage occurs, whole-body models relying on data-driven approaches, such as neural networks, typically require retraining with new data. This retraining process is time-consuming and can jeopardize mission success, as real-time motion prediction is only possible once the new training is complete. Additionally, the model's structural parameters may need to be reconfigured after damage. 
On the other hand, analytical modeling methods rely on fixed physical parameters and a predetermined structure, making them rigid to adjust in the event of unexpected damage. The challenges of complex manual calculations and reliance on physical measurements complicate the accurate modeling of damaged robots.

This paper addresses these issues by introducing several key contributions to enhance the flexibility and adaptability of current analytical methods for MLRs. Specifically, we propose an approach that overcomes the limitations of analytical modeling by enabling the model to adapt autonomously to morphological changes and damage scenarios without the need for manual recalculations. Our method maintains the precision of analytical models while incorporating modularity, allowing the system to handle a range of damage conditions more effectively than existing methods.

First, we present a model formulation that supports a wide range of MLR morphologies. This model leverages the Boltzmann-Hamel formulation of dynamics, offering a symbolic and singularity-free representation of the MLR dynamics that effectively decouples the main body and leg dynamics. Unlike previous works \cite{katayama2023model, liu2022design, carpentier2018multicontact, 9811926, batke2022optimizing, kontolatis2018investigation, lu2019dynamic, askari2019dynamic, mahkam2021framework}, we account for all six degrees of freedom (DoF) of the main body and their influence on leg motions. Our modeling method, based on symbolic equations of motion, achieves speeds that surpass traditional recursive models in efficiency (see Section III).

Second, we introduce a modular version of this model that can adapt to morphology changes. In related works  \cite{mahapatra_roy_pratihar_2020, 8743046, adak2022modeling, wensing2023optimization, biswal2021modeling}, MLR dynamics are typically modeled using Newton-Euler or Lagrangian methods, which require re-derivation of the dynamical equations to address significant, unforeseen morphology changes such as damage. In contrast, we derive a modular, symbolic, and singularity-free formulation that allows the MLR model to adjust dynamically to capture damage with minimal manipulation. Our method enables the modification of the model structure even in the middle of the simulation loop to account for unexpected physical damage to the legs (see Section \ref{chap:Modular}).

Third, we develop a formulation that not only reduces the computational load during runtime but also simplifies the generation of model equations without compromising generality. In this new formulation, the dynamical equations are derived only for the main body and a single leg. The model for the entire system, regardless of the number, geometry, and inertia of legs, can then be generated by replicating the leg equations. 
This leg reproduction capability, combined with our dynamic engine for simulating MLR motion in real-time, results in a model that is approximately three times faster than real-time (see Section \ref{chap:fast}).

Fourth, we provide extensive validation of our methods through physical experiments on both healthy and damaged robots under various damage scenarios. These validations incorporate both internal and external sensory information (see Section \ref{chap4:results}).

\subsection{Overview of Contents}

This paper presents a novel dynamic model formulation for MLR that meets all the above requirements. Section II reviews related work on MLR modeling and highlights the contributions of this study. Section III explains mathematical symbols and our previous theory on the singularity-free dynamical equations for MLR. The construction of our modeling method is divided into modular equations (Section IV) and fast modular dynamical equations (Section V). Finally, we validate our methods with the physical experiments on a hexapod robot, presented in Section VI. 

\section{Related work}

Scientists have thoroughly studied the dynamics of MLRs to develop their Guidance, Navigation, and Control (GNC) systems using two primary methods: (i) data-driven and (ii) analytic. Data-driven models, such as artificial neural networks \cite{an2023artificial}, spiking neural networks \cite{jiang2023fully}, reinforcement learning \cite{yang2020data}, and deep reinforcement learning \cite{gan2022energy}, necessitate a dataset for training. The effectiveness of these models is contingent upon the degree of similarity between the operational conditions of the system and the training setting. A robust training procedure for a robot typically lasts multiple days and may even be prolonged if multiple damage models are considered \cite{hwangbo2019learning, thor2020generic}.

For instance, Gan \textit{et al.} established a deep inverse reinforcement learning method for legged robots' terrain traversability modeling, which considers external and internal sensory information \cite{gan2022energy}. Similarly, researchers at Berkeley proposed a practical reinforcement learning method for precisely adjusting locomotion strategies in real-life scenarios. They showed that a small amount of practical training could significantly enhance the performance of a real A1 quadrupedal robot during its operation. This training allowed the robot to independently refine multiple locomotion skills in different environments, such as an outdoor lawn and various indoor terrains \cite{smith2022legged}. In addition, Yang \textit{et al.} created a hybrid model that combines a physics-based rolling spring-loaded inverted pendulum model with a data-driven model that uses Gaussian process regression to account for unmodeled dynamics. \cite{yang2020data, yang2021legged}.

Analytic models, such as the Newton-Euler equations \cite{mahapatra_roy_pratihar_2020, lin2001dynamic, bennani1996dynamic} and the Lagrangian/Hamiltonian formulation \cite{8743046, adak2022modeling}, are more complex to derive but do not necessitate a training phase to yield system predictions. These models demonstrate the ability to apply our knowledge of physics to unexpected scenarios, which makes them attractive for deployment in autonomous GNC systems \cite{wensing2023optimization}. Moreover, they are compatible with model-based controllers, providing benefits, e.g., incorporating safety restrictions, which are essential in MLR applications. Lagrangian approaches are particularly well-suited for constructing motion planners and controllers, making them stand out among diverse analytical methods. The authors in \cite{roy2011estimation,roy2013dynamic} developed the Euler-Lagrange equations of motion (EOMs) for a six-legged robot to investigate stability and estimate foot forces. Their approach assumes complete access to the states of the main body, which is moving at a constant height with a specified velocity. The model specifically focuses on the dynamics of the legs. Separate research conducted on the MIT Cheetah involved the implementation of a hierarchical controller that utilized proprioceptive impedance control. This approach simplified the dynamical equations by assuming that the robot's motion was planar \cite{hyun2014high}. Furthermore, the research paper by Biswal \textit{et al.} \cite{biswal2021modeling} employed the Euler-Lagrange model to enhance the distribution of forces on the feet of a quadruped robot.

Since MLRs have many degrees of freedom, their motion occurs in a high-dimensional state space, rendering their kinematic and dynamic models complex. It is, therefore, unrealistic to rely entirely on such complex models for implementing controllers, e.g., real-time model predictive control with a long prediction horizon \cite{katayama2023model}. Considering the computational efficiency of trajectory optimization and control, a variety of methods have also been developed that use experimentally simplified models such as centroidal dynamics \cite{orin2013centroidal, liu2022design, carpentier2018multicontact}, single rigid body \cite{9811926, aceituno2017simultaneous, winkler2018gait, batke2022optimizing}, or the linear inverted pendulum models \cite{kajita1991study, herdt2010online, caron2016zmp, li2021trajectory, iqbal2022drs}. For instance, a spring-loaded inverted pendulum model was used to describe the behavior of an MLR \cite{kontolatis2018investigation, lu2019dynamic}. Another study analyzed a dynamic model for miniature four-legged robots that is generated based on the Newton-Euler equation with massless legs \cite{askari2019dynamic}. Similarly, a framework was designed for the dynamic modeling of legged modular miniature robots by assuming no mass for the legs and using closed-chain leg's kinematic \cite{mahkam2021framework}.

In recent years, a combination of data-driven and analytical modeling methods, known as physics-inspired learning methods, has emerged. 
Such methods, including physics-informed neural networks \cite{willard2020integrating, wang2021physics}, act as universal function approximators that embed knowledge of physical laws governing a given dataset into the learning process. These methods can be described by partial differential equations \cite{TORABIRAD2020109687} and are increasingly employed in robotic dynamic modeling \cite{greydanus2019hamiltonian, raissi2018deep, zhong2019symplectic}. For instance, Saviolo \textit{et al.} developed a novel physics-inspired temporal convolutional network to learn quadrotor system dynamics purely from robot experience, resulting in accurate modeling crucial for ensuring agile, safe, and stable navigation. In another study, an augmented deep Lagrangian network was used to model the dynamics of robotic manipulators. This network effectively learned the inverse dynamics model of two multi-DoF manipulators, including a 6-DoF UR-5 robot and a 7-DoF SARCOS manipulator under uncertainties \cite{wu2024dynamic}. Despite the proficiency of these physics-enforced networks in modeling ideal physical systems, they encounter limitations when applied to systems with uncertain non-conservative dynamics that encounter significant morphology changes. Consequently, these methods are not yet suitable for modeling robotic systems undergoing physical damage.

Although the reviewed models partially address nonlinearity and coupling effects between the swing/support legs and the main body, their simplifying assumptions can restrict their effectiveness in abnormal situations. For example, in a robot with mechanical damage and a nonfunctional leg, the assumption of purely 2D motion for the main body is no longer valid. In these instances, a comprehensive whole-body model is essential for identifying and mitigating damage, as discussed in \cite{adak2022modeling}. Additionally, reduced models are often used to simplify the problem's dimensionality and computational complexity in locomotion control. However, it is difficult to account for whole-body constraints at the reduced level and to define an acceptable trade-off between tracking the reduced solution and searching for a new one \cite{budhiraja2019dynamics}. 

Several methods have attempted to capture the whole-body dynamics of multi-legged robotic systems \cite{mahapatra2020multi} using Newton-Euler formalism \cite{ouezdou1998dynamic}, Euler-Lagrange equations \cite{adak2022modeling, shah2012modular}, principle of virtual work\cite{liu2021dynamic}, or Sequential Linear Quadratic (SLQ) optimization \cite{neunert2017trajectory, neunert2018whole}. The SLQ algorithm initially estimates the dynamics of the system. It then linearizes the non-linear dynamics around the trajectory and solves the linear-quadratic optimum control problem in a backward manner. In \cite{bellicoso2017dynamic}, a whole-body model was used in a zero-moment point-based motion planner and hierarchical whole-body controller, which optimizes the whole-body motion and contact forces by solving a cascade of prioritized tasks. Their model utilized Hamiltonian unit quaternions to parametrize the orientation of the main body. In another work, Shah \textit{et al.} proposed a general framework for dynamic modeling and analysis of modular robots using the concept of kinematic modules, where each module consists of a set of serially connected links \cite{shah2012modular}. Their method results in recursive algorithms for inverse and forward dynamics using Decoupled Natural Orthogonal Complement (DeNOC) matrices \cite{nandihal2022dynamic}. Similarly, Sirois developed a whole-body nonlinear Model Predictive Control (MPC) for a free-flying space manipulator, employing the DeNOC-based methodology in the dynamic equations \cite{sirois2021nonlinear}. Current methods face limitations due to singularities arising from using Euler angles to parameterize the main body's orientation. Although these singularities are manageable under normal conditions, they become problematic in scenarios involving extreme rotations, such as a robot limping from mechanical damage. An alternative approach using quaternions avoids these singularities but introduces issues like the double-covering phenomenon and the need for complex and computationally expensive Jacobian mappings to convert quaternions to rotation matrices \cite{muller2016geometric}.

Symbolic EOMs facilitate numerical simulations and enable deeper mathematical analysis of the system. They allow for parametric studies, which provide symbolic expressions for equilibrium points and their stability conditions, assessing the effects of small changes in parameters (e.g., mass, length, and inertia) on system dynamics, and understanding how different coordinate systems influence problem complexity or introduce configuration singularities \cite{lind2021introduction}. This symbolic approach also enhances compatibility with other software, enabling multi-domain simulations, hardware-in-the-loop testing, and optimization or optimal control applications. However, despite these advantages, most multibody system software relies on numeric methods due to the mathematical complexity of deriving symbolic equations for high-dimensional nonlinear systems \cite{gede2013constrained}.


The use of symbolic equations in system dynamics dates back to the 1980s. Cheng \textit{et al.} derived symbolic dynamic EOMs for robotic manipulators using the Newton-Euler formalism, leveraging the Piogram symbolic method \cite{cheng1988symbolic}. Similarly, Fisette \textit{et al.} generated fully symbolic EOMs for multibody systems with flexible beams, presenting examples that demonstrate the efficiency of this approach \cite{fisette1997fully}. As mentioned earlier, symbolic EOMs can also facilitate simulation software development. For instance, \cite{gede2013constrained} developed a Python-based software package for constrained multibody dynamics, extending the SymPy computer algebra system. Furthermore, symbolic modeling is not limited to analytical methods; recent research has employed symbolic genetic algorithms to derive open-form partial differential equations from data, enabling model discovery without prior knowledge of the equation's structure \cite{chen2022symbolic}.

The Special Euclidean group in three dimensions, SE(3), is a matrix Lie group that globally describes the pose of a rigid body. This concept is founded on screw theory, originally developed by Ball based on Chasles' 19th-century work \cite{ball1998treatise}. Brockett used the exponential map of SE(3) to develop the Product of Exponentials (POE) for the forward kinematics of rigid multi-body systems \cite{brockett1984robotic}. Murray and colleagues expanded this by creating a Lie group framework for the kinematics, dynamics, and control of fixed-base manipulators, incorporating the geometric elements of screw theory \cite{murray1994mathematical}. Park emphasized the computational and analytical advantages of this approach by integrating the POE into recursive Newton-Euler and Euler-Lagrange dynamical equations, showing that it simplifies models without losing computational efficiency \cite{park1994computational}. Screw theory and the POE have been widely used in robotics for modeling, control, and path planning \cite{kim2019robotic, stramigioli2001modeling, moghaddam2023singularity}. By employing Lie groups, issues associated with Euler angles or quaternions, such as singularities, double covering, and stability, are avoided, allowing for a global representation of EOMs without local parameterization \cite{murray1994mathematical}. This approach has led to the development of several singularity-free dynamic modeling techniques, including the Euler-Poincaré equations \cite{marsden1993reduced, 10374259}, reduced Euler-Lagrange equations \cite{scheurle1993reduced, murray1997nonlinear}, and Boltzmann-Hamel equations \cite{duindam2007lagrangian}. Stramigioli and others have formulated Boltzmann-Hamel equations for both holonomic and non-holonomic multi-body systems on the SE(3) group, providing globally valid equations through a local diffeomorphism between SE(3) and its tangent space \cite{stramigioli2000hamiltonian, from2012corrections}. These singularity-free equations have been applied across various fields to describe the dynamics of underwater, space, and vehicle-manipulator systems \cite{from2010singularity, from2010, from2011}. Lie group techniques have been successfully used in MLRs for modeling and control purposes. For instance, Teng \textit{et al.} developed a geometric error-state MPC for tracking control of legged robotic systems evolving on SE(3). By exploiting the existing symmetry of the pose control problem on SE(3), they showed that the linearized tracking error dynamics and EOMs in the Lie algebra are globally valid and evolve independently of the system trajectory \cite{teng2022error}. In another work, a geometric bipedal robot model and controller were designed by extending commonly used non-linear control design techniques to non-Euclidean Lie-group manifolds. Then, they combined a model-based gait library design and deep learning to yield a near constant-time and constant-memory policy for fast, stable, and robust bipedal robot locomotion \cite{siravuru2020geometric}.

Matlab Simscape, Adams, and ANVEL are commercial programs for simulating dynamic mobile robots \cite{rohde2009interactive}. Although these programs can model vehicles such as trucks and cars, they are inappropriate for MLRs with arbitrary leg morphologies and contact dynamics. Furthermore, their computational rates are sufficient for offline assessment but insufficient for real-time control and planning \cite{seegmiller2016high}. Notable open-source dynamics libraries include the Open Dynamics Engine (OpenDE) \cite{smith2005open}, open-source Rigid Body Dynamics Library \cite{rbdl, bellicoso2017dynamic}, and OpenSYMORO (for Python symbolic robot modeling) \cite{khalil2014opensymoro}. Because of its accessible documentation and API, OpenDE is widely used to simulate MLRs dynamics \cite{cully2015robots, bongard2006resilient}. However, its contact point model and inflexibility in changing the morphology are its primary drawbacks when it comes to MLR modeling.

\section{Preliminaries and Problem Statement}
\label{preliminaries}

\subsection{Special Euclidean Lie Group Notation }
In this section, we provide a brief overview of the Special Euclidean Lie group $\SE$ and its associated Lie algebra $\se$, which represent the space of all possible configurations and velocities of a rigid body. Our notation follows \cite{murray1994mathematical} to which we refer readers for further details.

Consider a coordinate frame $\mathcal{X}_b$ attached to a rigid body. The configuration of the body relative to a reference frame $\mathcal{X}_a$ is expressed as an element of $\SE$:

\begin{align}
{g}_{a,b} = \begin{bmatrix}
{R}_{a,b} & {P}_{a,b} \\
\0_{1\times 3} & 1
\end{bmatrix} \in \SE,
\end{align}
where ${P}_{a,b} \in \R^3$ is the position vector, ${R}_{a,b} \in \SO$ is the rotation matrix in the Special Orthogonal grop, and $\0$ is a zero matrix of appropriate dimensions.

Given a rigid motion described by a curve $g_{a,b}(t) \in \SE$, the velocity of the rigid body relative to $\mathcal{X}_a$ can be observed either in the body frame $\mathcal{X}_b$ (referred to as body velocity) or in the reference frame $\mathcal{X}_a$ (referred to as spatial velocity). These velocities are obtained via left or right translation into the Lie algebra, respectively: $\widehat{\xi}_{a,b}^b = g_{a,b}^{-1} \dot{g}_{a,b} \in \se$ or $\widehat{\xi}_{a,b}^a = \dot{g}_{a,b} g_{a,b}^{-1} \in \se$.

The $\wedge$ operator serves as a vector space isomorphism between $\se \subset \R^{4 \times 4}$ and $\R^6$, such that for any $\xi = \left[\begin{smallmatrix} v \\ \omega\end{smallmatrix}\right] \in \R^6$, where $v, \omega \in \R^3$ correspond to the linear and angular velocities, we have:

\begin{align}
\widehat{\xi} = \begin{bmatrix} \widehat{\omega} & v \\ \0_{1\times 3} & 0 \end{bmatrix} \in \se.
\end{align}
Here, $\widehat{\omega} \in \mathbb{R}^{3 \times 3}$ denotes a skew-symmetric matrix such that $\widehat{\omega}x = \omega \times x $ for all $\omega, x \in \mathbb{R}^3$. The $\vee$ operator is defined as the inverse of the $\wedge$ operator.

Elements of $\se$ represent infinitesimal motions in time and can be mapped to transformations in $\SE$ using the exponential map $\exp: \se \to \SE$. This map integrates a constant twist over a unit of time to produce a transformation in $\SE$, and it can be used to globally parameterize one-degree-of-freedom (1-DoF) joints \cite{murray1994mathematical}. The exponential map for $\SE$ can be computed using the Rodriguez formula, as detailed in \cite{ExpMap, murray1994mathematical}.

The Adjoint operator $\Ad_{g_{a,b}}: \R^6 \to \R^6$ is used to transform twists (elements of $\se$) from the frame $\mathcal{X}_b$ to the frame $\mathcal{X}_a$, based on the transformation $g_{a,b}$:

\begin{align}
\Ad_{g_{a,b}} :=
\begin{bmatrix}
R_{a,b} & \widehat{P}_{a,b} R_{a,b} \\
\0_{3 \times 3} & R_{a,b}
\end{bmatrix}.
\end{align}
Finally, the Lie bracket operator $[\cdot, \cdot]$ defines an adjoint operator, denoted by $\ad_{\xi}: \R^6 \to \R^6$, such that for any $\xi, \eta \in \R^6$, we have $\ad_{\xi}(\eta) = [\xi,\eta] = (\widehat{\xi} \widehat{\eta} - \widehat{\eta} \widehat{\xi})^\vee$. For $\xi = \left[\begin{smallmatrix} v \\ \omega \end{smallmatrix}\right]$, this operator takes the form:

\begin{align}
\ad_\xi :=
\begin{bmatrix}
\widehat{\omega} & \widehat{v} \\
\0_{3 \times 3} & \widehat{\omega}
\end{bmatrix}.
\end{align}

Based on this mathematical notation, in our recent publication \cite{farghdani2024singularity}, we presented a theory for deriving singularity-free whole-body equations for an MLR that will be presented in the next section.

\subsection{Problem Statement}

An MLR consists of a main body indexed by $b$, $N$ legs indexed by $i\in\{1, \,...\, ,N\}$, and the $i^{th}$ leg includes $n_i$ links indexed by $j\in\{1, \,...\, ,n_i\}$, so that $l_{ij}$ represents Link $j$ of Leg $i$. Leg $i$ is composed of $n_i$ single-DoF joints, where Joint $j$ connects Body $j$ and $j-1$. The relative configuration manifold of a 1-DoF joint, considered in this work, is a 1-parameter subgroup of $\SE$ that is parameterized by the Lie algebra elements, using the exponential map of $\SE$~\cite{murray1994mathematical}. To each Joint $j$ in Leg $i$, we associate the joint parameter $\theta_{ij}\in\R$ and the twist $\xi_{ij}\in\R^6$ that corresponds to the axis of the relative screw motion between Body $j$ and $j-1$ and observed in the initial configuration of the main body. Note that we consider different degrees of freedom for the legs to model mechanical damage (morphological changes) in the system. Without loss of generality, only a single-branch leg kinematics is considered. The total number of robot's leg DoF is then $N_T:=\sum_{i=1}^{N}n_i$; hence, the system has $6+N_T$ DoF. 

At the Center of Mass (CoM) of the main body, we attach the body-fixed coordinate frame $\mathcal{X}_{b}$, and we reserve the subscript \textit{s} for the spatial coordinate frame. To complete our notation, at the CoM of every Link $l_{ij}$ and Leg $i$'s tip, we attach the frames $\mathcal{X}_{ij}$ and $\mathcal{X}_{it}$, respectively. We coin the constant generalized mass matrices of the main body and every Link $l_{ij}$ in their local frames by $I_b$ and $I_{ij}$, respectively. Let $\theta=[\theta_1^T,\ldots,\theta_N^T]^T\in\R^{N_T}$ be the collection of the leg joint parameters, such that $\theta_i = [\theta_{i1},\ldots,\theta_{in_i}]^T \in\R^{n_i}$ for Leg $i$, and $g_{s,b}\in\SE$ be the configuration of the main body. Then, the collection $Q=(g_{s,b},\theta)$ describes the configuration of the multi-legged robotic system. A rigid multi-body system, as studied in this paper, is visualized in Figure \ref{CHAP4:Leg_Frames}. \par

\begin{figure}[hbt!]
    \centering
    \includegraphics[width=0.5\textwidth]{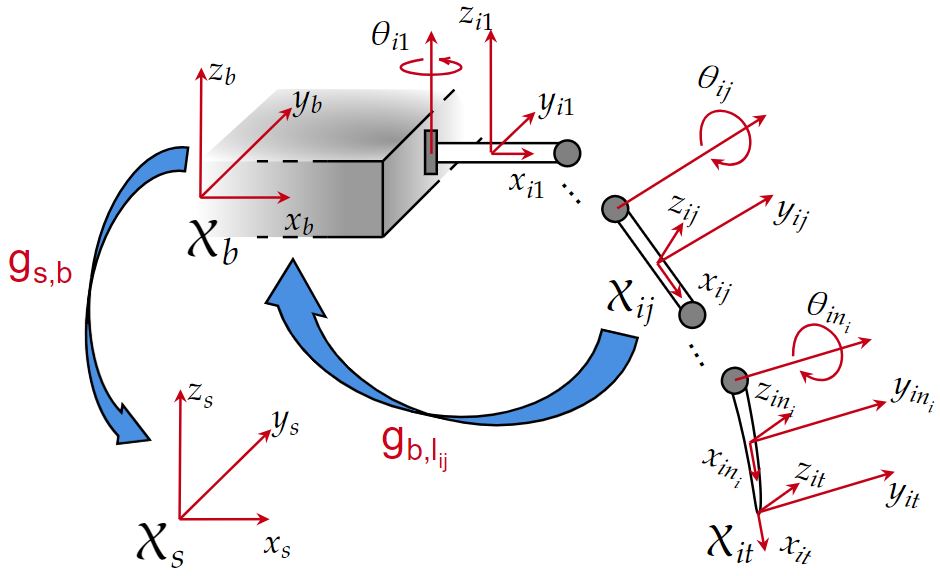}
    \caption{Frames assignment for the main body and Leg $i$}
    \label{CHAP4:Leg_Frames}
\end{figure}
The velocities of the bodies in the system are fully determined given the body velocity $V^b_{s,b}$ of the main body and the joint velocities $\dot\theta$, i.e., the quasi-velocity vector $v=[(V^b_{s,b})^T\,\,\dot\theta^T]^T\in\mathbb R^{6+N_T}$. For every Link $l_{ij}$, we have
\begin{align}\label{eq:Vsl}
    V_{s,l_{ij}}^b:=V_{s,b}^b+J_{ij}(\theta)\dot\theta,
\end{align}
defining the Jacobian $J_{ij}\in\R^{6\times N_T}$ corresponding to $l_{ij}$. This Jacobian calculates the relative velocity of Link $l_{ij}$ with respect to the main body frame and express in the same frame. Therefore, the full state space of the system contains the elements in the form of $(Q,v)$. 

\begin{theorem}[Boltzmann-Hamel equations of legged robots \cite{farghdani2024singularity}] \label{theorem1}
Consider an MLR with a 6-DoF main body, whose configuration is locally parameterized by $\Phi(g_{s,b},\phi):= g_{s,b}e^{\hat\phi}$ (such that $\phi\in\R^{6}$) at $g_{s,b}$, and $N$ legs, each with $n_i$ degrees of freedom. Let $S_b$ define the relation between the local velocities $\dot\phi$ and the body velocity $V^b_{s,b}$ and let $F_t=[(F_{1t}^{1t})^T\,\cdots\, (F_{Nt}^{Nt})^T]^T\in\R^{6 N}$ be the collection of all body wrenches applied at the tip of the legs. The singularity-free Boltzmann-Hamel equations of the system are given by:
\begin{gather}
    M(\theta)\dot v + C(\theta,v) v + \N(Q)= \tau + \mathcal{J}^TF_t, \label{chap4:eq_dyn}
\end{gather}
such that $\begin{bmatrix}
        \tau_b\\ \tau_\theta
    \end{bmatrix}:=\tau\in\R^{6+N_T}$ is the applied forces collocated with the quasi-velocities, the mass matrix 
\begin{align}\nonumber
    &\begin{bmatrix} M_{bb} & M_{b\theta} \\ M_{\theta b} & M_{\theta \theta}\end{bmatrix}\coloneqq M(\theta)  \\ 
    &~~~~~~~~~~~~~~~=\begin{bmatrix} I_b & \0_{6\times N_T} \\ \0_{N_T\times 6} & \0_{N_T\times N_T} \end{bmatrix} + \sum_{i=1}^{N} \sum_{j=1}^{n_i} M_{ij}(\theta),      \label{chap4:M}
\end{align}

\begin{align}\nonumber
 &\begin{bmatrix}
    M^{bb}_{ij} & M^{b\theta}_{ij} \\
    M^{\theta b}_{ij} & M^{\theta \theta}_{ij} \end{bmatrix} := M_{ij}(\theta)   \\  &=\!\!\begin{bmatrix}
  \Ad_{g_{b,l_{ij}}}^{-T}\!  I_{ij}  \Ad_{g_{b,l_{ij}}}^{-1} & \Ad_{g_{b,l_{ij}}}^{-T}\!  I_{ij}  \Ad_{g_{b,l_{ij}}}^{-1}  J_{ij} \\
  J_{ij}^T  \Ad_{g_{b,l_{ij}}}^{-T}\!  I_{ij}  \Ad_{g_{b,l_{ij}}}^{-1} & J_{ij}^T  \Ad_{g_{b,l_{ij}}}^{-T} \! I_{ij}  \Ad_{g_{b,l_{ij}}}^{-1}  J_{ij}
\end{bmatrix}\!\!, \label{chap4:M_link}
\end{align}
and the forces resulted from a potential field $U$ are
\begin{align}\label{chap4:N}
      \begin{bmatrix}
        \N_b\\ \N_\theta
    \end{bmatrix}:=\N(Q)=&\begin{bmatrix}
        S_b^{-T}\left.\frac{\partial U}{\partial\phi}\right|_{\phi=\0}\\ \frac{\partial U}{\partial\theta}
    \end{bmatrix}.
\end{align}
Introducing the momentum vector $ \mathcal{P} = \begin{bmatrix}
    \mathcal{P}_v\\\mathcal{P}_\omega\\\mathcal{P}_\theta
\end{bmatrix}:= Mv$, where $\mathcal{P}_v,\mathcal{P}_\omega\in\R^3$ and $\mathcal{P}_\theta\in\R^{N_T}$, the Coriolis matrix takes the following form
\begin{align}\nonumber
       &\begin{bmatrix} C_{bb} & C_{b\theta} \\ C_{\theta b} & C_{\theta \theta}\end{bmatrix}:=C(\theta,v) \, = \, \sum_{i=1}^{N}\sum_{k=1}^{n_i}\frac {\partial M }{\partial \theta_{ik} } \,\dot\theta_{ik} \\ 
 &~~~~~~~~~~~~~~~~~~~~~~~~~
 -\frac{1}{2} \, \begin{bmatrix}
  2 \tiny\begin{bmatrix} \0_{3\times 3} & \widehat{\mathcal{P}}_v \\ \widehat{\mathcal{P}}_v & \widehat{\mathcal{P}}_\omega \end{bmatrix} & \0_{6\times N_T}\\
  \frac {\partial ^T [\mathcal{P}^T_v~\mathcal{P}_\omega^T]^T }{\partial \theta }  & \frac {\partial ^T \mathcal{P}_\theta}{\partial \theta } 
\end{bmatrix}, \label{chap4:C}
\end{align}
and the Jacobian 
\begin{align} \label{jac}
    \mathcal{J}=\begin{bmatrix}
        \Ad_{g_{b,1t}}^{-1} & \Ad_{g_{b,1t}}^{-1} J_{1t}\\ \vdots & \vdots \\ \Ad_{g_{b,Nt}}^{-1} & \Ad_{g_{b,Nt}}^{-1} J_{Nt}
    \end{bmatrix}\in\R^{6N\times(6+N_T)}.
\end{align}
Here, $g_{b,it}\in\SE$ is the relative configuration of a frame attached to the tip of Leg $i$, and $J_{it}\in \mathbb{R}^{6 \times N_T}$ is the Jacobian for the same tip frame that provides its velocity relative to the main body and expressed in $\mathcal{X}_b$.
\end{theorem}


\begin{problem}

    A key limitation of the whole-body dynamical equation \eqref{chap4:eq_dyn} is its dependency on the robot’s morphology. When the robot's morphology changes, i.e., a link or leg is added or removed, due to, e.g., mechanical damage, these equations must be entirely rederived. This is beyond simply removing zero rows/columns or contact forces; the structural matrices, such as $M$ and $C$, must also be rederived that involve symbolic differentiation. Therefore, a computational algorithm is necessary to autonomously reshape structural matrices without manual differentiation, enabling real-time adaptation to instantaneous morphological changes caused by mechanical damage. To address this, we introduce a fast modular reformulation of \eqref{chap4:eq_dyn} in the following sections.
\end{problem}

\section{Modular Dynamical Equations of MLRs}\label{chap:Modular}

This section studies the symbolic derivation of the modular whole-body dynamics of MLRs. The concept of modularity discovered in the presented Boltzmann-Hamel Lagrangian dynamics entails the ability to remove modules within the system and to re-derive the corresponding symbolic model without necessitating additional differential manipulations. In this work, individual links and legs are regarded as modules. Our construction takes advantage of the coupled equations of the healthy MLR to rederive the equations using only basic operations such as matrix addition, subtraction, and reshape to capture link or leg loss in the MLR. 

\begin{remark}\label{rem:morph}    
In MLRs, we deal with a special morphological type of connection in the multibody system, i.e., a number of legs considered as serial links branched from a single rigid body (main body). This assumption in morphological type results in special properties of \eqref{chap4:eq_dyn} that are studied in the following sections. 
\end{remark}


\subsection{Modular Mass Matrix for MLRs}
In this section, we symbolically derive a modular form of the mass matrix $M(\theta)$. We start by exploring different properties of the elementary terms appearing in \eqref{chap4:M}.

\begin{lemma} \label{chap4:lemma1}
    For an MLR, the matrix $\Ad_{g_{b,l_{ij}}}$ corresponding to Link $l_{ij}$ appearing in \eqref{chap4:eq_dyn} is only dependent on Leg $i$'s joint states, i.e., {$\theta_i$}. More specifically it is a function of \unboldmath $\theta_{ik}$ for $0 < k \leqslant j$.
\end{lemma}
\begin{proof}
Using the exponential joint parameterization, the relative configuration of every leg link $l_{ij}$ can be obtained via the PEO formula \cite{murray1994mathematical}:
\begin{align}
    &g_{b,l_{ij}}({\theta_i})=e^{\widehat{\xi}_{i1} \theta_{i1}} \,\cdots\, e^{\widehat{\xi}_{ij} \theta_{ij}}g_{b,l_{ij}}(0).
\end{align}
Based on this equation, $g_{b,l_{ij}}$ and subsequently  $\Ad_{g_{b,l_{ij}}}$ are only related to the joint states of each leg, i.e., $\theta_i$. 
\end{proof}
\begin{lemma} For an MLR, the Jacobian
    ${J}_{ij} \in \mathbb{R}^{6 \times N_T}$ corresponding to Link $l_{ij}$ appearing in \eqref{chap4:eq_dyn} takes the following form
    \begin{align}\label{chap4:Jij}
    &{J}_{ij} (\theta_i) = \begin{bmatrix} \0_{6\times n_1} \cdots J^{b}_{b,l_{ij}} (\theta_i) \cdots \0_{6\times n_N} \, \end{bmatrix}\in\R^{6 \times N_T}.
\end{align}
    Here $J^{b}_{b,l_{ij}} (\theta_i)\in\R^{6\times n_i}$ is only dependent on Leg $i$'s joint states, i.e., {$\theta_i$}. More specifically, it is a function of \unboldmath $\theta_{ik}, ~0 < k \leqslant j$.
    \label{chap4:lemma2}
\end{lemma}
\begin{proof}
Let the twist $V_{s,l_{ij}}^b\in\R^6$ be the velocity of link $l_{ij}$ in the spatial frame and observed in $\mathcal{X}_b$. Based on \eqref{eq:Vsl}
\begin{equation}
    V^b_{s,l_{ij}} = V^b_{s,b} + V^b_{b,l_{ij}} = V^b_{s,b} + J_{ij}\dot\theta,
\end{equation}
where ${J}_{ij}$ provides the velocity of this link relative to the main body. According to Remark \ref{rem:morph}, 
\begin{align}
    &{J}_{ij} (\theta_i) = \begin{bmatrix} \0_{6\times n_1} \cdots J^{b}_{b,l_{ij}} (\theta_i) \cdots \0_{6\times n_N} \, \end{bmatrix}\in\R^{6 \times N_T},\\
    &J^{b}_{b,l_{ij}} (\theta_i) = [\xi_{i1} ~ \, \xi^{'}_{i2} \dotsb ~ \xi^{'}_{ij} \,\0_{6\times (n_i-j)}]\in\R^{6 \times n_i},
\end{align}
such that
\begin{equation}
     \xi^{'}_{ik} = \Ad_{(e^{\widehat{\xi}_{i1} \theta_{i1}} ~ \dotsb ~e^{\widehat{\xi}_{ik-1} \theta_{i(k-1)}} )} \xi_{ik} ~~~~~~~~~  k \leq j
\end{equation}
is the $k^{th}$ instantaneous joint twist in the frame $\mathcal{X}_{b}$ and $\xi_{ik}$ is constant~\cite{murray1994mathematical}. 
\end{proof}
\par
\begin{remark}
    The form of the Jacobian in \eqref{chap4:Jij} plays a pivotal role in our subsequent results regarding modularity in dynamical equations. According to this equation, not only each leg's Jacobian is a function of the corresponding leg's states, but it also contains nonzero columns only relating the same leg's joint speeds. As a result, the velocity twist of each leg will be fully decoupled from those of the other legs. The coupling appears indirectly by inclusion of the main body velocity in \eqref{eq:Vsl}.
\end{remark}

\begin{corollary}
    The block matrices of the mass matrix $M_{ij} \in \mathbb{R}^{(6 + N_T)\times(6 + N_T)}$ corresponding to Link $l_{ij}$, as defined in \eqref{chap4:M_link}, take the following form:
    \begin{align}
    M^{b\theta}_{i j}\!\!&=\!\!(M^{\theta b}_{i j})^T\!\!\! = \!\!\left[\begin{matrix} \0_{6 \times \sum\limits_{\kappa = 1}^{i-1}\!\!\! n_\kappa}
         & \M^{b\theta}_{i j} & \0_{6 \times \!\!\!\sum\limits_{\kappa = i+1}^{N}\!\!\!\!\! n_\kappa} \end{matrix}\right]\in\R^{6 \times N_T}\!\!, \\
    \M^{b\theta}_{i j} & =M_{ij}^{bb} J_{b,l_{ij}}^b
    =: \begin{bmatrix} [\mathfrak{M}^{b\theta}_{i j} ]_{6 \times j} & \0_{6 \times (n_i - j)} \end{bmatrix}\in\R^{6 \times n_i} ,
\end{align}
\begin{align}
    M^{\theta\theta}_{i j} &=
        \begin{bmatrix}
            \0_{\sum\limits_{\kappa = 1}^{i-1}\!\!\! n_\kappa \times \sum\limits_{\kappa = 1}^{i-1}\!\!\! n_\kappa} & \0 & \0\\
            \0& \mathcal{M}^{\theta \theta}_{i j} & \0 \\
            \0& \0 & \0\!\!\!_{\sum\limits_{\kappa = i+1}^{N}\!\!\!\!\! n_\kappa \times\!\!\! \sum\limits_{\kappa = i+1}^{N}\!\!\!\!\! n_\kappa}
        \end{bmatrix}\in\R^{N_T \times N_T}, \\\nonumber
    \mathcal{M}^{\theta \theta}_{i j} & = (J_{b,l_{ij}}^b)^T  M_{ij}^{bb}  J_{b,l_{ij}}^b\\ &=: \begin{bmatrix}
                [\mathfrak{M}^{\theta \theta}_{i j}]_{j \times j} & \0 \\ \0 & \0_{(n_i - j) \times (n_i - j)}
            \end{bmatrix}\in\R^{n_i \times n_i},
\end{align}
where $\M^{b\theta}_{i j}$ and $\M^{\theta \theta}_{i j}$ are only dependent on $\theta_i$, or more specifically, all \unboldmath $\theta_{ik}, ~0 < k \leqslant j$. Their nonzero parts are denoted by $\mathfrak{M}^{b\theta}_{i j}$ and $\mathfrak{M}^{\theta \theta}_{i j}$, respectively.
    \label{chap4:lemma3}
\end{corollary}
\begin{proof}
From \eqref{chap4:M_link}, $M_{ij}$ consists of three terms: $\Ad_{g_{b,l_{ij}}}, J_{ij}$ and $I_{ij}$. According to {Lemmas} \ref{chap4:lemma1} and \ref{chap4:lemma2}, for Leg $i$, $\Ad_{g_{b,l_{ij}}}$ and $J_{ij}$ are only function of {$\theta_i$} and \unboldmath$I_{ij}$ is constant. The explicit form of the block matrices is the direct consequence of the form of the Jacobian in \eqref{chap4:Jij}.
\end{proof}
\begin{remark}
    The form of individual mass matrices presented in Corollary \ref{chap4:lemma3} will be useful for further derivation of the modular representation of the Mass matrix. 
\end{remark} 
Before proceeding with the derivation of the modular form of the dynamical equations, it is essential to define a morphology vector, which serves to indicate the presence or absence of a leg or link within our formulation. 
\begin{definition}
    We define the \textbf{link morphology vector} $\EX$ for an MLR as:
\begin{align}
    &\EX := \begin{bmatrix} \EX_{leg_1} & \cdots & \EX_{leg_i} & \cdots & \EX_{leg_N}      
    \end{bmatrix}\in\R^{N_T},  \\
    &\EX_{leg_i}= \begin{bmatrix} \ex_{i1} & \cdots & \ex_{ij} & \cdots & \ex_{in_i}      
    \end{bmatrix} \in\R^{n_i}, \\ 
    &\ex_{ij}=\left\{ \begin{array}{rcl}
     1 & \mbox{if the link is present} \\ 
     0 & \mbox{if the link is absent  } 
    \end{array}\right..
    \end{align}
    Similarly, we define the \textbf{leg morphology vector} $\bar\EX$ by:
    \begin{align} 
    &\bar\EX := \begin{bmatrix} \bar\ex_1 & \cdots & \bar\ex_i & \cdots & \bar\ex_N      
    \end{bmatrix}\in\R^{N},  \\
    &\bar\ex_{i}=\left\{ \begin{array}{rcl}
     1 & \mbox{if} ~\sum\limits_{j = 1}^{n_i}\ex_{ij}>0 \\ 
     0 & \mbox{otherwise} 
    \end{array}\right..
\end{align}
\end{definition}
\noindent The vector $\EX$ is a binary vector, divided into $N$ segments corresponding to the robot’s legs. In the absence of damage, all elements of $\EX$ are one, and if a link is missing, the corresponding element is set to zero. We name each binary element $\ex_{ij}$ the \textit{link existence number} for $l_{ij}$. Similarly, the binary vector $\bar\EX$ consists of the \textit{leg existence numbers} $\bar\ex_i$ for all legs. In the following, we use the discovered modularity of the mass matrix to explicitly implement the link/leg existence numbers into the dynamical equations. This allows us to perform real-time damage simulations, i.e., changing the morphology of an MLR within the simulation loop. 


\begin{proposition}\label{proposition1}
    The modular mass matrix of an MLR, including the leg existence numbers is derived as:
    \begin{gather} 
        M(\theta)\!\!=\!\! 
        \small\begin{bmatrix}
            I_b\!\! +\!\! \sum\limits_{i=1}^{N} \bar\ex_{i} M^{bb}_{leg_i}\!\!(\theta_i)  & \bar\ex_{1}\mathcal{M}^{b\theta}_{leg_1}\!\!(\theta_1)  & \cdots & \bar\ex_{N}\mathcal{M}^{b\theta}_{leg_N}\!\!(\theta_N)  \\
            \bar\ex_{1}\mathcal{M}^{\theta b}_{leg_1}\!\!(\theta_1)  &  \bar\ex_{1}\mathcal{M}^{\theta\theta}_{leg_1}\!\!(\theta_1) & \0 & \0 \\
            \vdots & \0 & \ddots & \0 \\
            \bar\ex_{N}\mathcal{M}^{\theta b}_{leg_N}\!\!(\theta_N)  & \0 & \0 & \bar\ex_{N}\mathcal{M}^{\theta\theta}_{leg_N}\!\!(\theta_N)  
        \end{bmatrix} \label{chap4:m_robot_NEW}
    \end{gather}
    where we have defined the modular mass matrix for an individual Leg $i$, including the link existence numbers, by
     \begin{align}
        &\begin{bmatrix}
            M^{bb}_{leg_i} & \0 &  \mathcal{M}^{b\theta}_{leg_i} & \0 \\
            \0 & \0 & \0 & \0 \\
            \mathcal{M}^{\theta b}_{leg_i} & \0 &  \mathcal{M}^{\theta\theta}_{leg_i} & \0 \\
            \0 & \0 & \0 & \0 
        \end{bmatrix}\!\!\!\coloneqq\!\! M_{leg_i}(\theta_i) \!\!\coloneqq\!\! \sum\limits_{j=1}^{n_i} \ex_{ij} M_{ij}(\theta_i)\!. \label{m_leg_NEW}
    \end{align}
\end{proposition}
\begin{proof}
     According to Lemmas \ref{chap4:lemma1} \& \ref{chap4:lemma2} and Corollary \ref{chap4:lemma3}, the modular mass matrix of an individual Leg $i$, including the link existence numbers, takes the form:
    \begin{align*}
        &\begin{bmatrix}
            M^{bb}_{leg_i} & \0 &  \mathcal{M}^{b\theta}_{leg_i} & \0 \\
            \0 & \0 & \0 & \0 \\
            \mathcal{M}^{\theta b}_{leg_i} & \0 &  \mathcal{M}^{\theta\theta}_{leg_i} & \0 \\
            \0 & \0 & \0 & \0 
        \end{bmatrix}= M_{leg_i}(\theta_i)\\ &~~~~~~~~~ = \sum\limits_{j=1}^{n_i} \ex_{ij} M_{ij}(\theta_i) =
        \sum\limits_{j=1}^{n_i} \ex_{ij}\begin{bmatrix}
            M^{bb}_{ij} & \0 &  \mathcal{M}^{b\theta}_{ij} & \0 \\
            \0 & \0 & \0 & \0 \\
            \mathcal{M}^{\theta b}_{ij} & \0 &  \mathcal{M}^{\theta\theta}_{ij} & \0 \\
            \0 & \0 & \0 & \0 
        \end{bmatrix}, 
    \end{align*}
    where the leg level blocks of the modular mass matrix are denoted by $M^{bb}_{leg_i}$, $\mathcal{M}^{b\theta}_{leg_i}$, $\mathcal{M}^{\theta b}_{leg_i}$, and $\mathcal{M}^{\theta\theta}_{leg_i}$. Substituting this equation in \eqref{chap4:M}, we arrive at \eqref{chap4:m_robot_NEW}.
\end{proof}
\begin{remark}
    We incorporate two levels of modularity in the MLR's mass matrix \eqref{chap4:m_robot_NEW}: leg level and link level. If an entire leg is removed from the MLR, the corresponding leg existence number is set to 0; and if a link is removed from a leg, the corresponding link existence number is set to 0 while the leg existence number is 1.
\end{remark}

\subsection{Modular Coriolis Matrix for MLRs}

The Coriolis matrix of an MLR was introduced in \eqref{chap4:C}. Utilizing the mass matrix in \eqref{chap4:m_robot_NEW}, we aim to derive a closed form modular version of the Coriolis matrix, in this section. To this end, we first derive the closed form equations for the partial derivatives of the block components of the individual mass matrix of each link.

To simplify our notation, we define $_i\Add^k_j \in \mathbb R^{6\times6}$ as the accumulated Adjoint transformations given by the function:
\begin{equation}\label{Ad_new}
    _i\Add^k_j= \,  \left\{ \begin{array}{cl}
    \Ad^{-1}_{(e^{\widehat{\xi}_{ik} \theta_{ik}} \text{  } \dotsb \text{  }e^{\widehat{\xi}_{ij} \theta_{ij}} )}, &  0 < k \leq j \\ \1_{6\times6} & k = 0
    \\  \0_{6\times6}  & k > j
    \end{array}\right. 
\end{equation}
that can be substituted in \eqref{chap4:M_link} to arrive at
\begin{align}
      M_{ij}(\theta_i) \!\!=\!\! 
\begin{bmatrix}
  (_i\Add^1_j)^T \cI_{ij} \,_i\Add^1_j & (_i\Add^1_j)^T \cI_{ij} \,_i\Add^1_j \, J_{ij} \\
  J_{ij}^T \, (_i\Add^1_j)^T \cI_{ij} \,_i\Add^1_j & J_{ij}^T \, (_i\Add^1_j)^T \cI_{ij} \,_i\Add^1_j \, J_{ij}
\end{bmatrix}, \label{Mij_new2}
\end{align}
where $\cI_{ij} \coloneqq \Ad_{g_{b,l_{ij}}(0)}^{-T} \, I_{ij} \, \Ad_{g_{b,l_{ij}}(0)}^{-1}\in\R^{6\times 6}$.

\begin{lemma}    \label{lem:3}
    The partial derivatives of the mass matrix corresponding to Link $l_{ij}$ with respect to the joint states of Leg $i$ can be calculated in closed form by ($k=1,\cdots,j$):
    \small
    \begin{align} 
        \frac {\partial  M_{ij}^{bb}}{\partial \theta_{ik} } &= \left( \frac {\partial \,_i\Add^1_j }{\partial \theta_{ik} } \right)^T \cI_{ij} \,_i\Add^1_j + (\,_i\Add^1_j)^T \cI_{ij} \frac {\partial  \,_i\Add^1_j }{\partial \theta_{ik} } , \\ \nonumber
        \frac {\partial  \M_{ij}^{b\theta}}{\partial \theta_{ik} } &= (\frac {\partial  \M_{ij}^{\theta b}}{\partial \theta_{ik} })^T =  
        \left( \frac {\partial  \,_i\Add^1_j }{\partial \theta_{ik} } \right)^T \cI_{ij} \,_i\Add^1_j  J_{b,l_{ij}}^b \\ 
        &+ (\,_i\Add^1_j)^T \cI_{ij} \frac {\partial  \,_i\Add^1_j }{\partial \theta_{ik} }  J_{b,l_{ij}}^b + (\,_i\Add^1_j)^T \cI_{ij} \,_i\Add^1_j \frac {\partial  J_{b,l_{ij}}^b }{\partial \theta_{ik} }, \\ \nonumber
        \frac {\partial  \M_{ij}^{\theta \theta}}{\partial \theta_{ik} } &= \left( \frac {\partial  J_{b,l_{ij}}^b }{\partial \theta_{ik} } \right)^T (\,_i\Add^1_j)^T \cI_{ij} \,_i\Add^1_j J_{b,l_{ij}}^b + \\ \nonumber
        & (J_{b,l_{ij}}^b)^T  \left( \frac {\partial  _i\Add^1_j }{\partial \theta_{ik} } \right)^T \cI_{ij} \,_i\Add^1_j  J_{b,l_{ij}}^b + \\ \nonumber
        &  (J_{b,l_{ij}}^b)^T  (\,_i\Add^1_j)^T \cI_{ij} \frac {\partial  \,_i\Add^1_j }{\partial \theta_{ik} }  J_{b,l_{ij}}^b  + \\
        &(J_{b,l_{ij}}^b)^T  (\,_i\Add^1_j)^T \cI_{ij} \,_i\Add^1_j \frac {\partial  J_{b,l_{ij}}^b }{\partial \theta_{ik} }.
    \end{align}
    \normalsize
    Here, all partial derivatives can be calculated from
    \begin{gather} \nonumber
        \frac {\partial \, J^{b}_{b,l_{ij}} }{\partial \theta_{ik} } = \begin{bmatrix}
            \0_{6\times1} \text{  } \, \frac {\partial \, \xi^{'}_{i2} }{\partial \theta_{ik} }  \dotsb \text{  } \frac {\partial \, \xi^{'}_{ij} }{\partial \theta_{ik} } \,\0_{6\times (n_i-j)}  \end{bmatrix} ,\\ 
            \frac {\partial  \xi^{'}_{i\beta} }{\partial \theta_{ik} } = \left\{ \begin{array}{ccl}
        (\,_i\Add^{1}_k)^{-1} \ad_{\widehat{\xi}_{ik}} (\,_i\Add^{k+1}_{\beta-1})^{-1} \xi_{i\beta}  & k < \beta -1 \\ 
         \0_{6 \times 6 }  & k\geq \beta -1
        \end{array}\right. \\
        \frac {\partial \, _i\Add^1_j }{\partial \theta_{ik} } = \,  \left\{ \begin{array}{ccl}
        -\,_i\Add^{k-1}_j \ad_{\widehat{\xi}_{ik}} \,_i\Add^{1}_k  & k < j \\ 
         \0_{6 \times 6 }  & k\geq j
        \end{array}\right. 
    \end{gather} 
\end{lemma}
\begin{proof}
    The proof is presented in Appendix \ref{app:mass}.
\end{proof}
\begin{remark}
    For Link $l_{ij}$, the partial derivative of the mass matrix $M_{ij}$ is zero for any $k>j$. 
\end{remark}

\begin{proposition}\label{proposition2}
    The modular Coriolis matrix of an MLR can be derived in closed form, such that
    \begin{align}
       C(\theta,v) = \mathcal{A} - \frac{1}{2} \begin{bmatrix} \mathcal{B} & \0 \\ \mathcal{D} & \mathcal{E} \end{bmatrix} \label{c_m}
\end{align}
where
\begin{align} \label{eq:cA}
    &\mathcal{A} = \sum\limits_{i=1}^{N}\bar\ex_i \sum\limits_{j=1}^{n_i} \left( \sum\limits_{k=1}^{j} \ex_{ij} \frac {\partial M_{ij}(\theta_i) }{\partial \theta_{ik} } \right)\,\dot\theta_{ik}, \\ \label{eq:cB}
    &\mathcal{B} = 2  \begin{bmatrix} \0_{3\times 3} & \widehat{\mathcal{P}_v} \\ \widehat{\mathcal{P}_v} & \widehat{\mathcal{P}_\omega} \end{bmatrix} , \\ \label{eq:cD}
    &\mathcal{D} = \frac {\partial ^T [\mathcal{P}^T_v~\mathcal{P}_\omega^T]^T }{\partial \theta }, \\ \label{eq:cE} 
    &\mathcal{E} = \frac {\partial ^T \mathcal{P}_\theta}{\partial \theta }, 
\end{align} 
and based on Lemma \ref{lem:3}, we compute 
\begin{gather} \label{eq:P}
    \mathcal{P} \!\!=\!\! \begin{bmatrix}
    \mathcal{P}_v\\\mathcal{P}_\omega\\\mathcal{P}_\theta
\end{bmatrix} \!\!=\!\! \begin{bmatrix}
        I_b V^b_{s,b} + \sum\limits_{i=1}^{N} \bar\ex_i\sum\limits_{j=1}^{n_i} \ex_{ij}( M_{ij}^{bb} V^b_{s,b} +  \ \mathcal{M}_{ij}^{b\theta} \dot\theta_i ) \\ \begin{matrix} \bar\ex_1\sum\limits_{j=1}^{n_1} \ex_{1j} \bigg( \mathcal{M}^{\theta b}_{1j} V^b_{s,b} + \mathcal{M}^{\theta \theta}_{1j} \dot\theta_1 \bigg) \\ \vdots \\ \bar\ex_N\sum\limits_{j=1}^{n_N} \ex_{Nj} \bigg( \mathcal{M}^{b\theta}_{Nj} V^b_{s,b} +  \mathcal{M}^{\theta \theta}_{Nj} \dot\theta_N \bigg) \end{matrix} \end{bmatrix},
\end{gather}
\begin{align}\label{eq:parP1}
    &\frac {\partial  [\mathcal{P}^T_v~\mathcal{P}_\omega^T]^T }{\partial \theta_{\kappa\beta} }= \bar\ex_\kappa \sum\limits_{j=1}^{n_\kappa} \ex_{\kappa j}(\frac{\partial M_{\kappa j}^{bb}}{\partial \theta_{\kappa\beta}}  V^b_{s,b} +  \frac{\partial \mathcal{M}_{\kappa j}^{b\theta}}{\partial \theta_{\kappa\beta}} \dot\theta_\kappa ),\\
    &\frac {\partial  \mathcal{P}_\theta}{\partial \theta_{\kappa\beta} }=\begin{bmatrix}
          \begin{matrix} \0_{n_i\times 1} \\ \vdots \\ \bar\ex_\kappa\sum\limits_{j=1}^{n_\kappa} \ex_{\kappa j} \bigg( \frac{\partial \mathcal{M}^{\theta b}_{\kappa j}}{\partial \theta_{\kappa\beta}}  V^b_{s,b} + \frac{\partial \mathcal{M}^{\theta \theta}_{\kappa j}}{\partial \theta_{\kappa\beta}} \dot\theta_\kappa \bigg)\\ \vdots\\\0_{n_N\times 1} \end{matrix} \end{bmatrix}.\label{eq:parP2}
\end{align}

\end{proposition}
\begin{proof}
    The proof is presented in Appendix \ref{app:PRO2}.
\end{proof}

\begin{remark}
    Note that the nonzero terms in \eqref{eq:parP1} and \eqref{eq:parP2} only appear when $i=\kappa$ and $\beta\le j$, due to the modularity discovered in the dynamical equations. This property considerably increases the speed of computation of the Coriolis matrix using the closed form equations in Lemma \ref{lem:3}. 
\end{remark}

\subsection{Forming the Dynamical Equations of a Damaged MLR}

In this section, we explain how the modularity introduced in the previous sections enables the rederivation of the equations of motion for a damaged MLR within a simulation loop, utilizing only matrix addition, subtraction, and reshaping, without requiring any differentiation operations. 
To achieve this, we begin by the healthy morphology of the MLR specifying the number of legs $N$ and the geometry of each leg ($n_i,\xi_{ij}$)---number of DoF and joint screws. We then input the existence numbers of links and legs. If the existence number of Leg $i$ is zero, the leg is no longer part of the MLR. Consequently, the first step is to remove $M_{leg_i}$ from the derivation of the MLR mass matrix and compute the updated mass matrix $M$. A similar process is applied if one or more links of Leg $i$ are missing. However, instead of removing the entire $M_{leg_i}$, only the corresponding $M_{ij}$ terms for the missing links are excluded, and the updated $M_{leg_i}$ is incorporated into the computation of $M$.  

The next step is to compute the updated Coriolis matrix $C$. This requires removing the derivatives of the missing $M_{ij}$ terms from the calculation of $\mathcal{A}$ in \eqref{eq:cA} and $\mathcal{P}$ in \eqref{eq:P}, followed by reconstructing the Coriolis matrix. Although the effects of missing links or legs are accounted for in the mass and Coriolis matrices, their dimensions must also be adjusted by eliminating the corresponding zero rows and columns. Finally, the remaining vectors in the dynamic equation \eqref{chap4:eq_dyn} must be updated by removing the rows corresponding to the missing links or legs. This process is summarized in Algorithm \ref{algo:modul}.  
\begin{algorithm} 
\caption{Model reconfiguration for morphology changes}
\begin{algorithmic}[1] \label{algo:modul}
    \REQUIRE The morphology of the healthy MLR, i.e., $N,n_i,\xi_{ij}$
\REQUIRE Morphology vectors $\bar\EX$ and $\EX$ 
\STATE Set $M,C,\mathcal{A}, \mathcal{B}, \mathcal{D}, \mathcal{E}, \mathcal{P}$ to zero with appropriate dimensions from the healthy morphology
\FOR{$i = 1$ to $N$} 
    \STATE $M_{leg_i}=\0$ 
    \IF{$\bar\ex_{i} = 1$} 
        \IF{$\EX_{leg_i}$ is a valid morphology}
            \FOR{$j = 1$ to $n_i$}
                \IF{$\ex_{ij}=1$}
                    \STATE Compute blocks of $M_{ij}$ based on Corollary \ref{chap4:lemma3}
                    \STATE $M_{leg_i}\leftarrow M_{leg_i} + M_{ij}$ based on \eqref{m_leg_NEW}
                    \STATE Update $\mathcal{P}, \mathcal{B}$ based on \eqref{eq:P}, \eqref{eq:cB}
                    \FOR{$k=1$ to $j$}
                        \STATE Compute $\frac{\partial M_{ij}}{\partial \theta_{ik}}$ based on Lemma \ref{lem:3}.
                        \STATE $\mathcal{A}\leftarrow \mathcal{A} + \frac{\partial M_{ij}}{\partial \theta_{ik}}$ based on \eqref{eq:cA}
                        \STATE Update $\mathcal{D}, \mathcal{E}$ based on \eqref{eq:cD}, \eqref{eq:cE}
                    \ENDFOR     
                    \STATE Update Coriolis matrix $C$ \eqref{c_m}
                \ENDIF
            \ENDFOR    
        \ELSE
           \STATE Error!
        \ENDIF
    \ENDIF
    \STATE Assemble $M$ based on \eqref{chap4:m_robot_NEW}
\ENDFOR
\STATE Reshape $M$ and $C$ by removing rows/columns corresponding to zero existence numbers
\STATE Remove rows of $\mathcal{J}$ corresponding to removed tip links
\STATE Reshape vectors $\N,v,\dot v,\tau,F_t$
\RETURN Reduced order rederived $M,C,\mathcal{J},\N,v,\dot v,\tau, F_t$
\end{algorithmic}
\end{algorithm}

\begin{remark}
    A morphology change can be accommodated by simply updating the morphology vectors to compute the updated modular mass and Coriolis matrices. This update automatically skips the computation of all terms related to the missing bodies, as the existence numbers are embedded in all steps of the derivation of the healthy model. Note that only the computations corresponding to the present links and legs are performed and no additional differentiation operations or manual modifications of the equations are required. This enables instantaneous model adjustments within the simulation loop. 
\end{remark}

\section{FAST Modular Dynamic Equations of Multi-Legged Robots}\label{chap:fast}

In the previous section, we derived a modular version of the dynamical equations for a given healthy MLR, allowing direct implementation of morphological changes into the equations. The main focus of this section is on constructing the model of a healthy MLR. The proposed methodology takes advantage of symbolic computation and symmetry in the geometry of the system modules to offer an inherently fast and modular set of dynamical equations. When combined with the equations from the previous section, this model can not only account for damage or loss of modules but also facilitate the construction of an MLR model by attaching similar legs at different locations of the main body. 

Let $\mathcal{L}=\{(n_\alpha,\xi_{\alpha j})|\alpha=1,\cdots, \bar N, j=1,\cdots,n_\alpha\}$ be a list of $\bar N$ unique leg morphologies (including number of leg DoF and joint screws) from which we select the types of legs connected to the MLR's main body. For every leg morphology $\alpha$ in $\mathcal{L}$ attached at an arbitrary symbolic location on the main body, we construct the symbolic version of the link mass matrices $M_{\alpha j}$, the matrix $\N_\alpha$, and the Jacobian $\mathcal{J}_\alpha$ from \eqref{chap4:M_link}, \eqref{chap4:N}, and \eqref{jac}. Note that $\N_\alpha$ and $\mathcal{J}_\alpha$ are the gravitational forces and the Jacobian exclusively corresponding to Leg $\alpha$. The structure of these symbolic matrices is independent of the leg attachment location, allowing them to be numerically evaluated within a simulation loop multiple times to capture multiple instances of a single leg morphology attached at various positions on the main body. It is important to note that the level of symbolic representation in the equations is arbitrary. By symbolically inserting the constant joint twists, the leg morphology itself can also vary within the simulation loop, provided that the number of leg DoF remains fixed. Therefore, a single symbolic computation of the structural matrices for a given leg morphology is sufficient to simulate the motion of the entire MLR.


\begin{remark}
    Note that we can one-by-one connect the legs to the main body and progressively construct the equations of motion, due to the modularity presented in the previous section. In this construction, we also need the assumption that the vector $\N$ corresponds to a constant gravitational field. 
\end{remark}
We have the following crucial property of the structural matrices.
\begin{lemma}
    For the leg morphology $\alpha$ from $\mathcal{L}$ with $n_\alpha$ revolute degrees of freedom, the matrices $M_{\alpha j}$, $\N_\alpha$, and $\mathcal{J}_\alpha$ contain linear combination of trigonometric functions of the joint states with the maximum order of $2n_\alpha$, i.e., all functions in the form of the members of  
 \begin{gather}
            \Gamma_\alpha= \Bigg\{ \nu : \R^{n_\alpha} \to \R ~\Bigg|~ \nu(\theta_\alpha) = \prod_{j=1}^{n_\alpha} \sigma_j(\theta_{\alpha j}); ~ \sigma_j \in Z^{2}(\theta_{\alpha j}) \Bigg\} ,
            \label{Gamma}
        \end{gather} 
such that
\begin{gather}\nonumber
            Z^{2}(\lambda) = \Bigg\{ \delta : \R \to [-1, 1] ~\Bigg|~ \delta(\lambda) = \sin(\lambda)^{\psi} \cos(\lambda)^{\beta}, \\ 
            \psi, \beta \in \mathbb{Z}^{\geq 0}, ~  \psi + \beta  \leqslant 2  \Bigg\} . 
        \end{gather}
\end{lemma}
\begin{proof}
    In our modular construction, Leg $i$ can be considered as a serial link manipulator, and the proof follows directly from \cite[Prop. 3]{lloyd2021numeric}.
\end{proof}

Based on this lemma, we can immediately decompose every $M_{\alpha j}$ into a constant matrix and a matrix of all allowable trigonometric functions in $\Gamma_\alpha$. More specifically, we have 
\begin{align}
 M_{\alpha j}(\theta_\alpha) =:\begin{bmatrix}
           \mathcal{I}^{bb}_{\alpha j} \mathcal{T}^{bb}_{\alpha j} & \0 &  \mathcal{I}^{b\theta}_{\alpha j} \mathcal{T}^{b\theta}_{\alpha j} & \0 \\
            \0 & \0 & \0 & \0 \\
            (\mathcal{I}^{b\theta}_{\alpha j} \mathcal{T}^{b\theta}_{\alpha j})^T & \0 &  \mathcal{I}^{\theta\theta}_{\alpha j} \mathcal{T}^{\theta\theta}_{\alpha j} & \0 \\
            \0 & \0 & \0 & \0 
        \end{bmatrix} \label{eq:fast_link m}
\end{align}
where we have $\mathcal{T}^{bb}_{\alpha j}=\diag_6\left(\F_\alpha \right) \in \R^{6 \gamma_\alpha \times 6} $ and  $\mathcal{T}^{b \theta}_{\alpha j}=\mathcal{T}^{\theta \theta}_{\alpha j}=\diag_{n_\alpha}\left(\F_\alpha \right) \in \R^{n_\alpha \gamma_\alpha \times n_\alpha } $, and
\begin{align} \label{f_mod1}
\F_{\alpha} &:= \begin{bmatrix}
        1 \\  1^{st}~\text{order functions} \\ 2^{nd}~\text{order functions} \\ \vdots \\ 2n_\alpha^{th}~\text{order functions}
    \end{bmatrix}\in\R^{\gamma_\alpha},
\end{align}
which contains all trigonometric functions present in $M_{leg_\alpha}(\theta_\alpha)$ starting by $1^{st}$ order terms to $2n_\alpha^{th}$ order terms. Note that the function $\diag$ generates a diagonal matrix from copies of the entry. For a given leg morphology $\alpha$, $\gamma_\alpha$ is dependent on the DoF of the leg. For instance, $\gamma_\alpha$ is equal to 6, 30, or 140 for one, two, or three DoF legs, respectively. Assuming that every element of $\F_\alpha$ is a linearly independent variable, we symbolically obtain the coefficient matrices 
$\mathcal{I}^{\theta \theta}_{\alpha j} \in \R^{n_\alpha \times n_\alpha \gamma_\alpha} $, $\mathcal{I}^{b \theta}_{\alpha j} \in \R^{6 \times \gamma_\alpha  n_\alpha} $, and $\mathcal{I}^{bb}_{\alpha j} \in \R^{6 \times 6\gamma_\alpha } $ by taking the partial derivative of columns of $M^{bb}_{\alpha j}$ with respect to $\F_{\alpha}$ as:
\begin{align} \label{i_mod1}
    \mathcal{I}^{bb}_{\alpha j} &=
    \begin{bmatrix}
        \frac {\partial \big[M^{bb}_{\alpha j}\big]_1 }{\partial \F_{\alpha} } & \cdots & \frac {\partial \big[M^{bb}_{\alpha j}\big]_6 }{\partial \F_{\alpha} }
    \end{bmatrix}\in \R^{6 \times 6\gamma_i} \\ \label{i_mod2}
    \mathcal{I}^{b \theta}_{\alpha j} &=
    \begin{bmatrix}
        \frac {\partial \big[\M^{b \theta}_{\alpha j}\big]_1 }{\partial \F_{\alpha} } & \cdots & \frac {\partial \big[\M^{b \theta}_{\alpha j}\big]_{n_\alpha} }{\partial \F_{\alpha} }
    \end{bmatrix}\in \R^{6 \times \gamma_\alpha  n_\alpha} \\ \label{i_mod3}
    \mathcal{I}^{\theta \theta}_{\alpha j} &=
    \begin{bmatrix}
        \frac {\partial \big[\M^{\theta \theta}_{\alpha j}\big]_1 }{\partial \F_{\alpha} } & \cdots & \frac {\partial \big[\M^{\theta \theta}_{\alpha j}\big]_{n_\alpha} }{\partial \F_{\alpha} }
    \end{bmatrix}  \in \R^{n_\alpha \times n_\alpha \gamma_\alpha}
\end{align}
where $ \big[\cdot]_k$ denotes the $k^{th}$ column of the entry matrix. 
We adopt a same procedure to calculate the decomposed version of the vector $\N$. As $\N_b$ in \eqref{chap4:N} is not a function of $\theta_\alpha$, we only need to calculate the decomposed $\N_\alpha$ for an arbitrary Leg $\alpha$: 
\begin{align} \label{eq:fast_N}
     \N_\alpha &= \mathcal{I}^{N}_{\alpha}  \F_{\alpha},  
\end{align}
where $\mathcal{I}^{N}_{\alpha} := \frac {\partial \N_\alpha }{\partial \F_{\alpha} }  \in \R^{n_\alpha \times \gamma_\alpha}$.
Finally, we need to calculate the decomposed version of $\mathcal{J}_\alpha$:
\begin{align} \label{eq:fast_J}
     \mathcal{J}_\alpha^T &= \mathcal{I}^{\mathcal{J}}_{\alpha} \mathcal{T}^{\mathcal{J}}_{\alpha}, 
\end{align}
where $\mathcal{I}^{\mathcal{J}}_{\alpha} := \begin{bmatrix}\frac {\partial \big[\mathcal{J}_\alpha^T\big]_1 }{\partial \F_{\alpha} } & \cdots & \frac {\partial \big[\mathcal{J}_\alpha^T\big]_6 }{\partial \F_{\alpha} } \end{bmatrix} \in \R^{6+N_T \times 6\gamma_\alpha}$.

The modular Coriolis matrix of an MLR was introduced in \eqref{c_m}. Here, we introduce a fast computational method (comparing to the closed form equations in Lemma \ref{lem:3}), using the mass decomposition in \eqref{eq:fast_link m}. 

\begin{remark}
    Leg $\alpha$'s vector of trigonometric functions $\F_\alpha$ contains only terms involving $\sin$ and $\cos$  functions with a maximum order of $2n_\alpha$. Therefore, the vectors of partial derivatives of $\F_\alpha$ with respect to $\theta_{\alpha k}$ can be symbolically calculated and implemented before the simulation loop.
\end{remark}


    Given the vectors of partial derivatives $\F_{\alpha k}:=\frac{\partial \F_\alpha}{\partial \theta_{\alpha k}}$ and According to Proposition \ref{proposition2}, the partial derivatives of the mass matrix in Lemma \ref{lem:3}, and hence, the Coriolis matrix can be calculated in closed form by the following computations:
    \begin{align}  \label{eq:pM1}
        \frac {\partial  M^{bb}_{\alpha j}}{\partial \theta_{\alpha k} } &= \mathcal{I}^{bb}_{\alpha j}\frac {\partial  \mathcal{T}^{bb}_{\alpha j}}{\partial \theta_{\alpha k} } = \mathcal{I}^{bb}_{\alpha j}\diag_6\left(\F_{\alpha k}\right), \\ \label{eq:pM2}
        \frac {\partial  \M_{\alpha j}^{b\theta} }{\partial \theta_{\alpha k} } &= (\frac {\partial  \M_{\alpha j}^{\theta b} }{\partial \theta_{\alpha k} })^T = \mathcal{I}^{b\theta}_{\alpha j}\frac {\partial  \mathcal{T}^{b\theta}_{\alpha j} }{\partial \theta_{\alpha k} }= \mathcal{I}^{b\theta}_{\alpha j}\diag_{n_\alpha}\left(\F_{\alpha k}\right), \\ 
        \frac {\partial  \M_{\alpha j}^{\theta \theta}}{\partial \theta_{\alpha k} } &= \mathcal{I}^{\theta\theta}_{\alpha j}\frac {\partial \mathcal{T}^{\theta\theta}_{\alpha j}}{\partial \theta_{\alpha k} }= \mathcal{I}^{\theta\theta}_{\alpha j}\diag_{n_\alpha}\left(\F_{\alpha k}\right). \label{eq:pM3}
    \end{align}
The symbolic decomposition process of the structural matrices for incorporating a generic leg morphology is outlined in Algorithm~\ref{algo:fast1}. This procedure can be iteratively applied to construct a repository of structural matrices corresponding to each unique leg morphology within the set~$\mathcal{L}$.

\begin{algorithm} 
\caption{Symbolic Decomposition of Structural Matrices}
\begin{algorithmic}[1] \label{algo:fast1}
\REQUIRE Symbolic matrices $\mathcal{J_\alpha},\N_\alpha$ for Leg $\alpha$ from the set $\mathcal{L}$ and $M_{\alpha  j}$ for all the links of that leg based on \eqref{chap4:M_link}, \eqref{chap4:N}, and \eqref{jac}. 
    \STATE Collect the symbolic vector $\F_\alpha$ in the form of \eqref{f_mod1}
    \STATE Symbolically derive matrices $\mathcal{I}^{N}_{\alpha}, \mathcal{I}^{\mathcal{J}}_{\alpha}, \mathcal{T}^{\mathcal{J}}_{\alpha}$
    \FOR{$j = 1$ to $n_\alpha$}
        \STATE  Symbolically derive matrices $\mathcal{I}^{bb}_{\alpha j}, \mathcal{T}^{bb}_{\alpha j}, \mathcal{I}^{b\theta}_{\alpha j} \mathcal{T}^{b\theta}_{\alpha j}, \mathcal{I}^{\theta\theta}_{\alpha j}$, $\mathcal{T}^{\theta\theta}_{\alpha j}$ based on \eqref{i_mod1}-\eqref{i_mod3}
        \STATE  Symbolically form $\F_{\alpha j}$                   
    \ENDFOR
\RETURN Symbolic matrices $\mathcal{I}^{bb}_{\alpha j}, \mathcal{T}^{bb}_{\alpha j}, \mathcal{I}^{b\theta}_{\alpha j} \mathcal{T}^{b\theta}_{\alpha j}, \mathcal{I}^{\theta\theta}_{\alpha j}$, $\mathcal{T}^{\theta\theta}_{\alpha j}$, $\mathcal{I}^{N}_{\alpha}, \mathcal{I}^{\mathcal{J}}_{\alpha}, \mathcal{T}^{\mathcal{J}}_{\alpha}$,$\F_\alpha$, and $\F_{\alpha j}$ 
\end{algorithmic}
\end{algorithm}
\noindent The decomposition process enables the development of a fast simulation engine capable of incorporating morphology changes during simulation. 
Inside a simulation loop, for each Leg $i\in\{1,\ldots,N\}$, the link mass matrices $M_{ij}$, their partial derivatives with respect to joint states $\frac{\partial M_{ij}}{\partial \theta_{ik}}$, the vector $\N_i$, and the Jacobian $\mathcal{J}_i$ are then computed numerically based on the specified leg morphology. This is done by substituting the leg-specific geometric parameters and joint states into the precomputed symbolic expressions for matices of Algorithm \ref{algo:fast1}. The remaining steps follow the same procedure described in Algorithm~\ref{algo:modul}. A key advantage of this method is the efficient computation of $\frac{\partial M_{ij}}{\partial \theta_{ik}}$. Rather than repeatedly performing matrix multiplications in the closed-form expressions derived in Lemma~\ref{lem:3}, the symbolic derivatives of $\F_i$ are computed once outside the simulation loop, and numerical values are substituted at runtime. This significantly reduces computational load. The complete procedure of morphology change using the presented fast modeling strategy is summarized in Algorithm~\ref{algo:fast2}.
\begin{algorithm} 
\caption{Fast model reconfiguration for morphology changes}
\begin{algorithmic}[1] \label{algo:fast2}
    \REQUIRE The morphology of the healthy MLR, i.e., $N,n_i,\xi_{ij}$
\REQUIRE Morphology vectors $\bar\EX$ and $\EX$ 
\REQUIRE Symbolic matrices $\mathcal{I}^{bb}_{\alpha j}, \mathcal{T}^{bb}_{\alpha j}, \mathcal{I}^{b\theta}_{\alpha j} \mathcal{T}^{b\theta}_{\alpha j}, \mathcal{I}^{\theta\theta}_{\alpha j}$, $\mathcal{T}^{\theta\theta}_{\alpha j}$, $\mathcal{I}^{N}_{\alpha}, \mathcal{I}^{\mathcal{J}}_{\alpha}, \mathcal{T}^{\mathcal{J}}_{\alpha}$,$\F_\alpha$, and $\frac{\partial \F_\alpha}{\partial \theta_{\alpha j}}$ for $\alpha=1,\cdots, \bar N$, and $j=1,\cdots,n_\alpha$
\STATE Set $M,C,\mathcal{A}, \mathcal{B}, \mathcal{D}, \mathcal{E}, \mathcal{P}$ to zero with appropriate dimensions from the healthy morphology
\FOR{$i = 1$ to $N$} 
    \STATE $M_{leg_i}=\0$ 
    \IF{$\bar\ex_{i} = 1$} 
        \IF{$\EX_{leg_i}$ is a valid morphology}
            \FOR{$j = 1$ to $n_i$}
                \IF{$\ex_{ij}=1$}
                    \STATE Compute blocks of $M_{ij}$ based on \eqref{eq:fast_link m}
                    \STATE $M_{leg_i}\leftarrow M_{leg_i} + M_{ij}$ based on \eqref{m_leg_NEW}
                    \STATE Update $\mathcal{P}, \mathcal{B}$ based on \eqref{eq:P}, \eqref{eq:cB}
                    \FOR{$k=1$ to $j$}
                        \STATE Compute $\frac{\partial M_{ij}}{\partial \theta_{ik}}$ based on \eqref{eq:pM1} to \eqref{eq:pM3}
                        \STATE $\mathcal{A}\leftarrow \mathcal{A} + \frac{\partial M_{ij}}{\partial \theta_{ik}}$ based on \eqref{eq:cA}
                        \STATE Update $\mathcal{D}, \mathcal{E}$ based on \eqref{eq:cD}, \eqref{eq:cE}
                    \ENDFOR     
                    \STATE Update Coriolis matrix $C$ \eqref{c_m}
                \ENDIF
            \ENDFOR  
            \STATE Compute $\N_{i}$ based on \eqref{eq:fast_N}
            \STATE Compute $\mathcal{J}^T_{i}$ based on \eqref{eq:fast_J}
        \ELSE
           \STATE Error!
        \ENDIF
    \ENDIF
    \STATE Assemble $M$ based on \eqref{chap4:m_robot_NEW}
    \STATE Assemble $\mathcal{J}$ based on \eqref{jac}
    \STATE Assemble $\N$ based on \eqref{chap4:N}
\ENDFOR
\STATE Reshape $M$ and $C$ by removing rows/columns corresponding to zero existence numbers
\STATE Remove rows of $\mathcal{J}$ corresponding to removed tip links
\STATE Reshape vectors $\N,v,\dot v,\tau,F_t$
\RETURN Reduced order rederived $M,C,\mathcal{J},\N,v,\dot v,\tau, F_t$
\end{algorithmic}
\end{algorithm}

\section{Simulation \& Experiment Results}\label{chap4:results}
We evaluate the performance of the proposed modeling method for simulating MLRs using the fast, modular dynamics introduced in Sections \ref{chap:Modular} and \ref{chap:fast}, along with a local joint space controller, contact model, and gait generation algorithm. As a representative case study, we focus on a six-legged robot modeled as a multibody system with six 3-DoF legs. The effectiveness of our model is assessed through multiple experiments, incorporating internal and external sensory data from a real six-legged robot and motion-tracking cameras, as detailed in the following sections.

\subsection{Platform Description}

The experimental setup incorporates a Hiwonder JetHexa, a six-legged robot with 3-Dof legs. Each leg has three serial bus servo motor-actuated joints. The onboard robot computer is an NVIDIA Jetson Nano Developer Kit with an MPU6050 IMU \cite{jethexa}. Figure \ref{Robot} shows a picture of the system. 
\begin{figure}[hbt!]
  \centering
  \includegraphics[width=.3\textwidth]{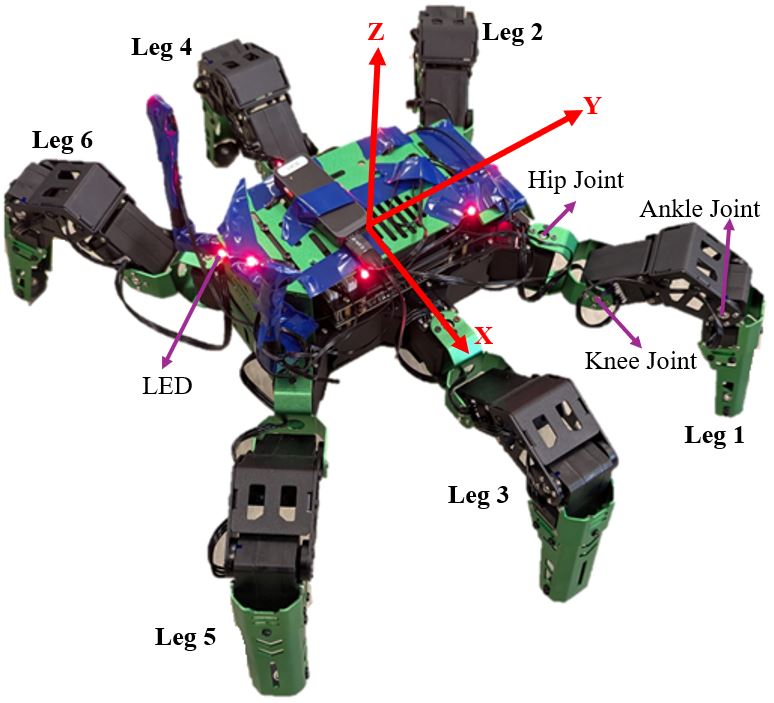}\\
  \caption{Hiwonder JetHexa robot with motion-tracking LEDs}
  \label{Robot}
\end{figure}
Tables \ref{trunk parameters} and \ref{leg parameters} show the inertial and geometric parameters of the simulated model based on the Hiwonder JetHexa robot. 
\begin{table}[ht]
\centering 
\caption {Physical parameters for the main body of the robot}
\label{trunk parameters} 
\scalebox{1} {
\begin{tabular}{c c c}
  \hline \hline
  Main body Parameters & Symbol & Value \\
   \hline
  Mass ($kg$) & $m_{b}$ & 1.35\\
  Main body dimensions ($10^{-3} ~m$) & $X_b$ & 84 \\
  & $Y_b$ & 198 \\
  & $Z_b$ & 36 \\
  Location of Leg 1 \& 2 hip joints ($10^{-3} ~m$) & $x_{h_{1/2}}$ & +/-51 \\
  & $y_{h_{1/2}}$ & 93 \\
  & $z_{h_{1/2}}$ & 0 \\
  Location of Leg 3 \& 4 hip joints ($10^{-3} ~m$) & $x_{h_{3/4}}$ & +/-73 \\
  & $y_{h_{3/4}}$ & 0 \\
  & $z_{h_{3/4}}$ & 0 \\
  Location of Leg 5 \& 6 hip joints ($10^{-3} ~m$) & $x_{h_{5/6}}$ & +/-51 \\
  & $y_{h_{5/6}}$ & -93 \\
  & $z_{h_{5/6}}$ & 0 \\
  Moment of inertia ($10^{-4} ~ kg.m^2 $) & $I_{xb}$ & 46\\
  & $I_{yb}$ & 9.36\\
  & $I_{zb}$ & 52\\
      \hline \hline
\end{tabular}  }
\end{table}
\begin{table}[ht]
\centering 
\caption {Physical parameters for each leg of the robot}
\label{leg parameters} 
\scalebox{1} {
\begin{tabular}{c c c c c}
  \hline \hline
  Link Parameters & Symbol & Link 1 & Link 2 & Link 3 \\
   \hline
  Mass ($kg$) & $m_{ij}$ & 0.02 & 0.07 & 0.11\\
  Length ($10^{-3} ~m$) & $L_{ij}$ & 45 & 77 & 123\\
  Moment of inertia  & $^{xx} I_{ij}$ & 1 & 0.23 & 0.22\\
  ($10^{-4} ~ kg.m^2 $) & $^{yy} I_{ij}$ & 8.28 & 3.07 & 10\\
  & $^{zz} I_{ij}$ & 9.09 & 2.91 & 10.01\\
      \hline \hline
\end{tabular}  }
\end{table}

This paper employs a general contact point model that considers tangential and normal force components at the contact interface of a leg tip with the ground. The normal contact forces are captured using a set of linear spring-damper equations once a contact is detected. For the tangential direction, viscous friction is considered to represent the forces parallel to the ground. Table \ref{contact parameters} lists the experimentally identified parameters to model contact between the tip of the legs and the ground. 
\begin{table}[ht]
\centering 
\caption {Contact point dynamics parameters}
\label{contact parameters} 
\scalebox{1} {
\begin{tabular}{c c c}
  \hline \hline
  Contact Parameters & Symbol & Value \\
   \hline
  Ground stiffness ($N/m$) & $k_{z}$ & 10000\\
  Ground viscosity ($N/(m/s)$) & $d_{z}$ & 150 \\
  Viscous friction coefficient ($N/(m/s)$) & $d_{t}$ & 50\\
      \hline \hline
\end{tabular}  }
\end{table}
In all experiments, we aim to walk with the Tripod gait in a straight line with a constant velocity of $0.1 ~m/s$. In this gait, there are always three feet in support phase and the other three feet in the swing state. Legs $\lbrace 1,4,5\rbrace$ are in the same phase (Group 1), and Legs $\lbrace2,3,6\rbrace$ are in the same phase (Group 2). The phase difference $\varphi$ between the two groups is $\pi$. The trajectory of the tip of the leg is designed relative to the main body for a straight motion in the y direction, as in Figure \ref{Robot}, while keeping the main body parallel to the ground. Therefore, the body is considered fixed, and the desired position of Leg $i$'s tip in the main body frame $\Bar{p}^b_{b,it}$ is defined by:
\begin{equation}
    \begin{cases}
      \Bar{p}^b_{b,it_x} = x_{h_{it}} = constant\\
      \Bar{p}^b_{b,it_y} = y_{h_{it}} -0.5L_{sl}\cos(k\frac{\pi}{0.5N_g} + \Phi_i(1-\frac{2}{N_g}) ) \\
       ^{swing}\Bar{p}^b_{b,it_z} = z_{h_{it}} + 0.5L_{sh}(1-\cos(k\frac{2\pi}{0.5N_g})) - H_0  \\
      ^{support}\Bar{p}^b_{b,it_z} = z_{h_{it}} + 0.5L_{sd}(1-\cos(k\frac{2\pi}{0.5N_g})) - H_0 
    \end{cases}
    \label{tip_ref}
\end{equation}
where $L_{sl}$, $L_{sh}$, and $L_{sd}$ denote the step length, height, and depth, respectively, $k =  1, ~ \dotsb ~, N_g$, and $H_0$ is the main body's initial height. $x_{h_{it}}$, $y_{h_{it}}$, and $z_{h_{it}}$ are the position of the first (hip) joint of Leg $i$ in the main body frame. $N_g$ is the number of points on the gait and will be calculated from the time step $\Delta t$ and the cycle time $T_g=N_g\Delta t$. Table \ref{gait parameters} shows the parameters of the tripod gait used in computer simulations.
\begin{table}[ht]
\centering 
\caption {Tripod gait parameters}
\label{gait parameters} 
\scalebox{1} {
\begin{tabular}{c c c}
  \hline \hline
  Trunk Parameters & Symbol & value \\
   \hline
  Step length ($10^{-2} ~m$) & $L_{sl}$ & 5 \\
  Step height ($10^{-2} ~m$) & $L_{sh}$ & 12 \\
  Step depth ($10^{-2} ~m$) & $L_{sd}$ & 0.1 \\
  Cycle time ($sec$) & $T_g$ & 1.3 \\  
  Initial height ($10^{-2} ~m$) & $H_{0}$ & 12 \\
      \hline \hline
\end{tabular}  }
\end{table}
The controller aims to control each leg relative to the body in joint space. Hence, a trajectory for the tip is generated from \eqref{tip_ref}, and desired joint trajectories can be found with the close-form of inverse kinematic equations for a 3-Dof manipulator. A joint space PID controller is designed to follow these trajectories with proportional, derivative, and integral gains $K_1 = 50$, $K_2 = 1$, and $K_3 = 10$, respectively. 


\subsection{Experiment Results}

In our experiments, we simulate the locomotion of a six-legged robot under two distinct damage scenarios using the dynamical equations described earlier. The first scenario involves the complete removal of Legs 3 and 4, representing damage to one leg in each leg group. The second scenario simulates damage to Legs 4 and 5, specifically targeting the loss of their last two links ($l_{42}, ~l_{43}, ~l_{52}, ~l_{53}$), which corresponds to two damaged legs within a single leg group. Although the robot maintains stability in both scenarios, its motion deviates significantly from the intended design. Using the same gait trajectory, controller, and initial configuration as in the previous section, we expect to observe noticeable tilting in the robot's movement. To have a reasonable validation of our model with sensory data, we compared the model's output with IMU and camera data for the main body orientation. However, for the robot's translational motion, we relied solely on camera data, as double integration from IMU acceleration data introduces significant errors and inaccuracies.
Figures \ref{damage1_o}, and \ref{damage1_p} shows how the main body orientation and CoM position change with respect to the spatial coordinate frame during the first damage scenario. Simultaneously staying in the y direction, the robot oscillates about the initial position in the X and Z directions. This can be seen from both CoM motion and the Euler angle about the Y (pitch) and Z (yaw) axes. The ZYX Euler angles are used to demonstrate main body orientation. The model's output is visualized using Matlab's animation toolbox, as shown in Figure \ref{damage1-snap}. Due to the damage, the robot's legs remain in contact with the ground and cannot lift off, resulting in a lack of forward motion. Consequently, we observe only the effect of the swing phase on the main body's motion.
\begin{figure}[hbt!]
  \centering
  \includegraphics[width=.45\textwidth]{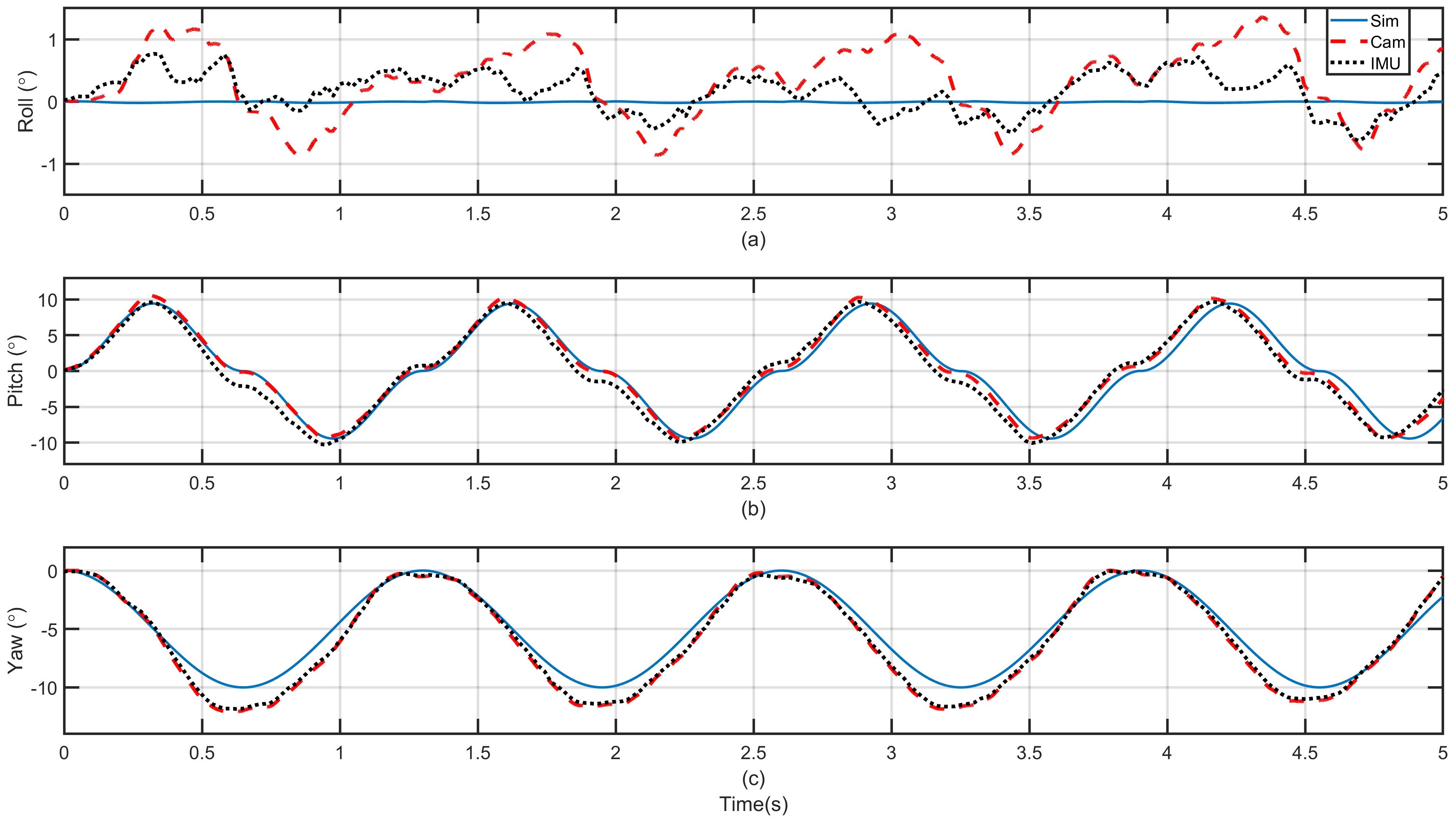}\\
  \caption{Main body orientation during the first damage scenario}
  \label{damage1_o}
\end{figure}
\begin{figure}[hbt!]
  \centering
  \includegraphics[width=.45\textwidth]{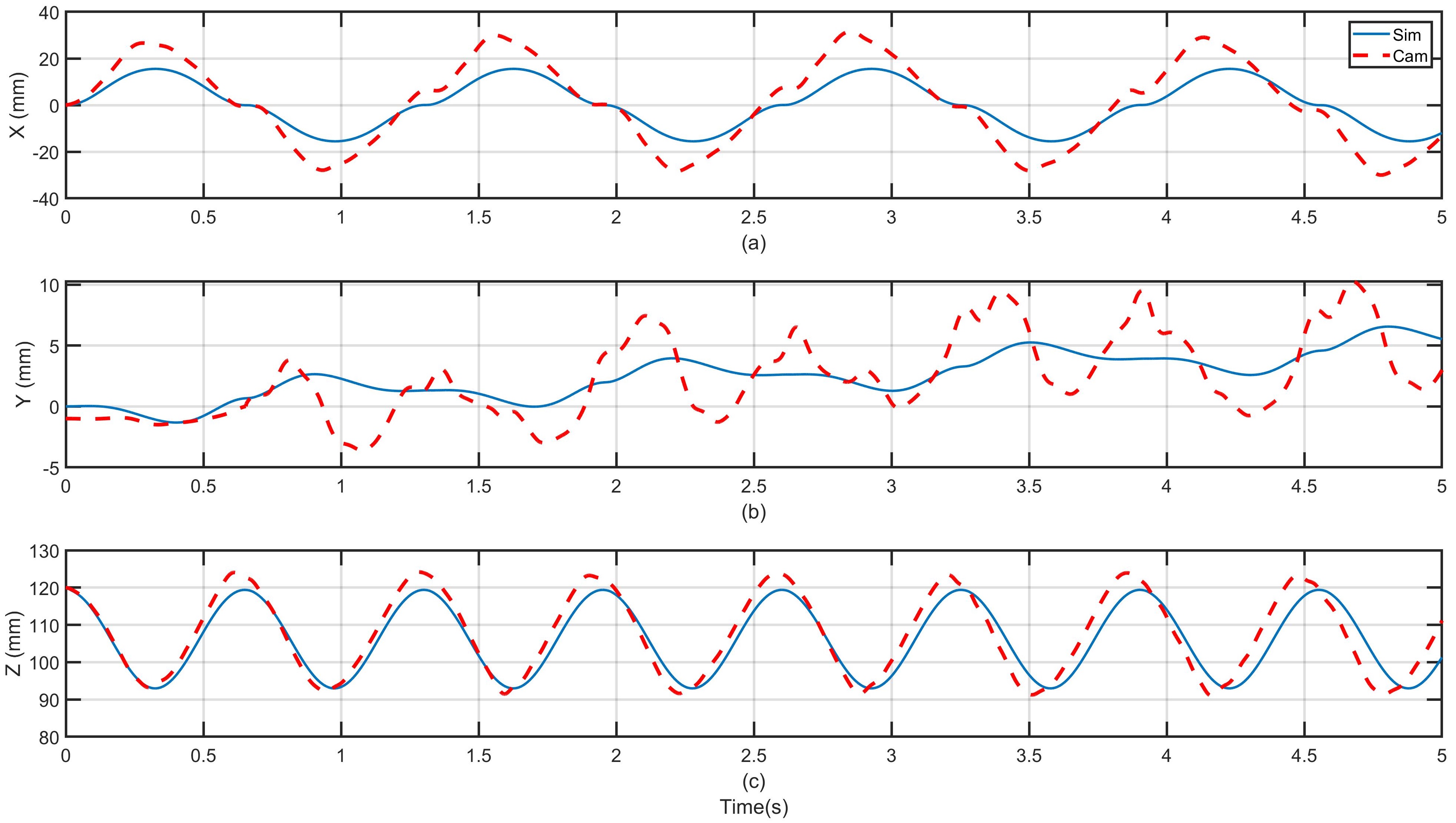}\\
  \caption{Main body COM position during the first damage scenario}
  \label{damage1_p}
\end{figure}
\begin{figure*}[htbp]
    \centering
    \includegraphics[width=.95\textwidth]{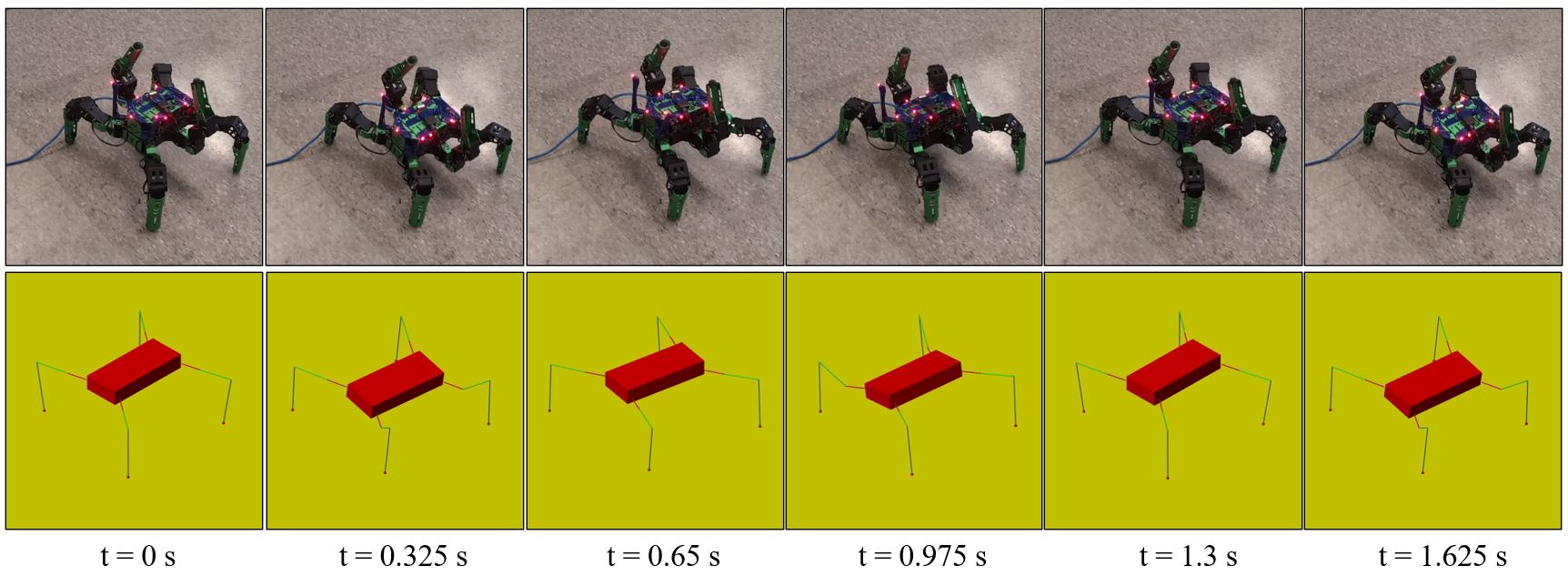}
    \caption{Snapshots of 1.625 seconds of model and robot motions for the first damage scenario}
    \label{damage1-snap}
\end{figure*}

Figures \ref{damage2_o}, and \ref{damage2_p} illustrate that despite the first damage scenario, in the second damage scenario, the robot exhibits forward motion in the Y direction and lateral motion in the X direction. With two legs in the first leg group ($\lbrace 1,4,5 \rbrace$) damaged, this group cannot maintain the robot's stability. As a result, there is a drop in the Z direction when this group is expected to stabilize the robot, disrupting the swing phase of the second group of legs ($\lbrace 2,3,6 \rbrace$). Furthermore, Leg 1, the only remaining healthy leg in the first group, undergoes a rotational motion about the Z axis during the support phase, resulting in motion in the X and Y directions. Figure \ref{damage2-snap} shows the snapshots of the model and robot motions for 1.625 seconds. Using a computer equipped with an Intel(R) Core(TM) i7-6700HQ CPU, our simulation engine completed a 5-second motion simulation in 1.93 seconds, which shows its real-time performance. 
\begin{figure}[hbt!]
  \centering
  \includegraphics[width=.45\textwidth]{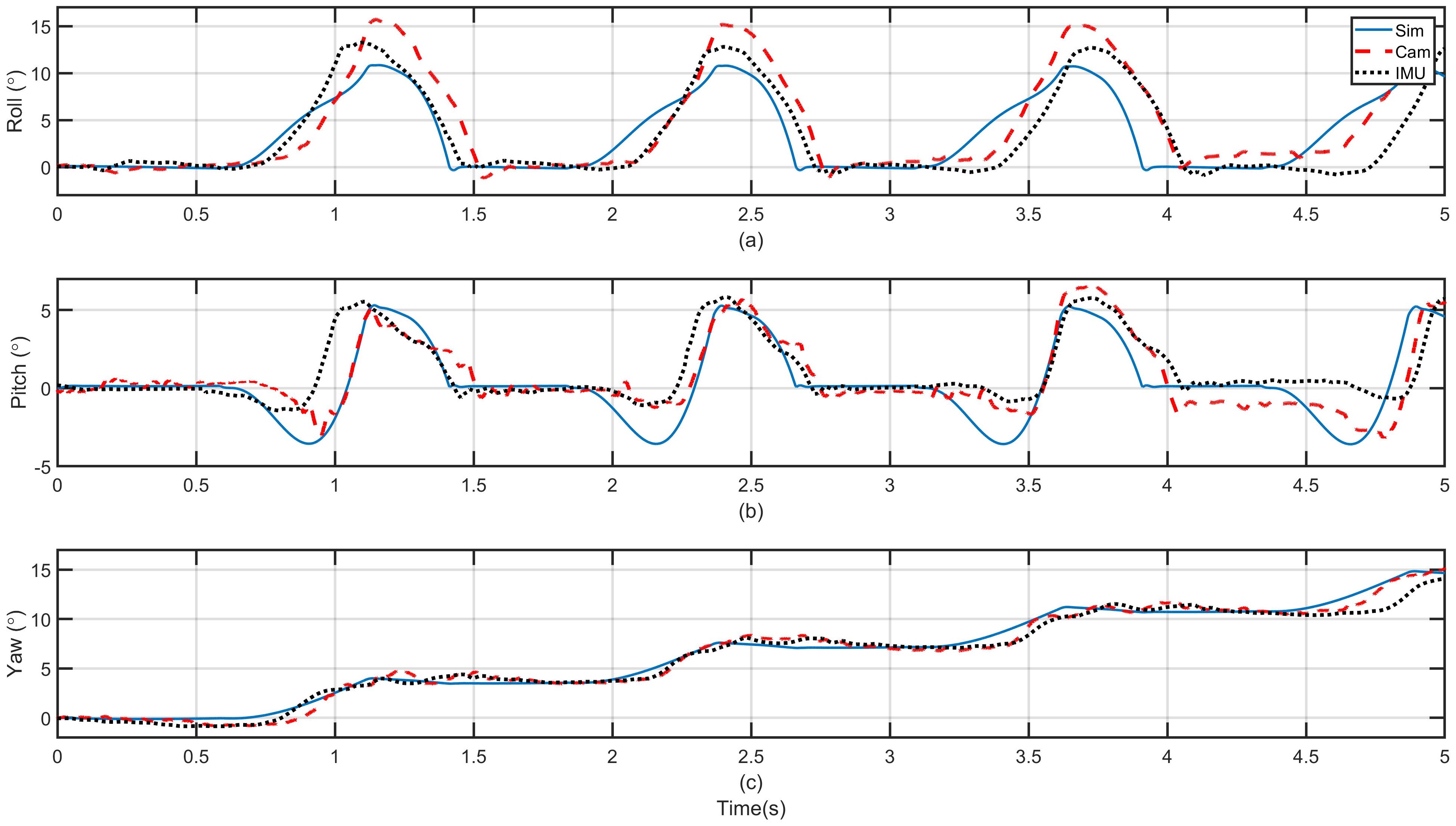}\\
  \caption{Main body orientation during the second damage scenario}
  \label{damage2_o}
\end{figure}
\begin{figure}[hbt!]
  \centering
  \includegraphics[width=.45\textwidth]{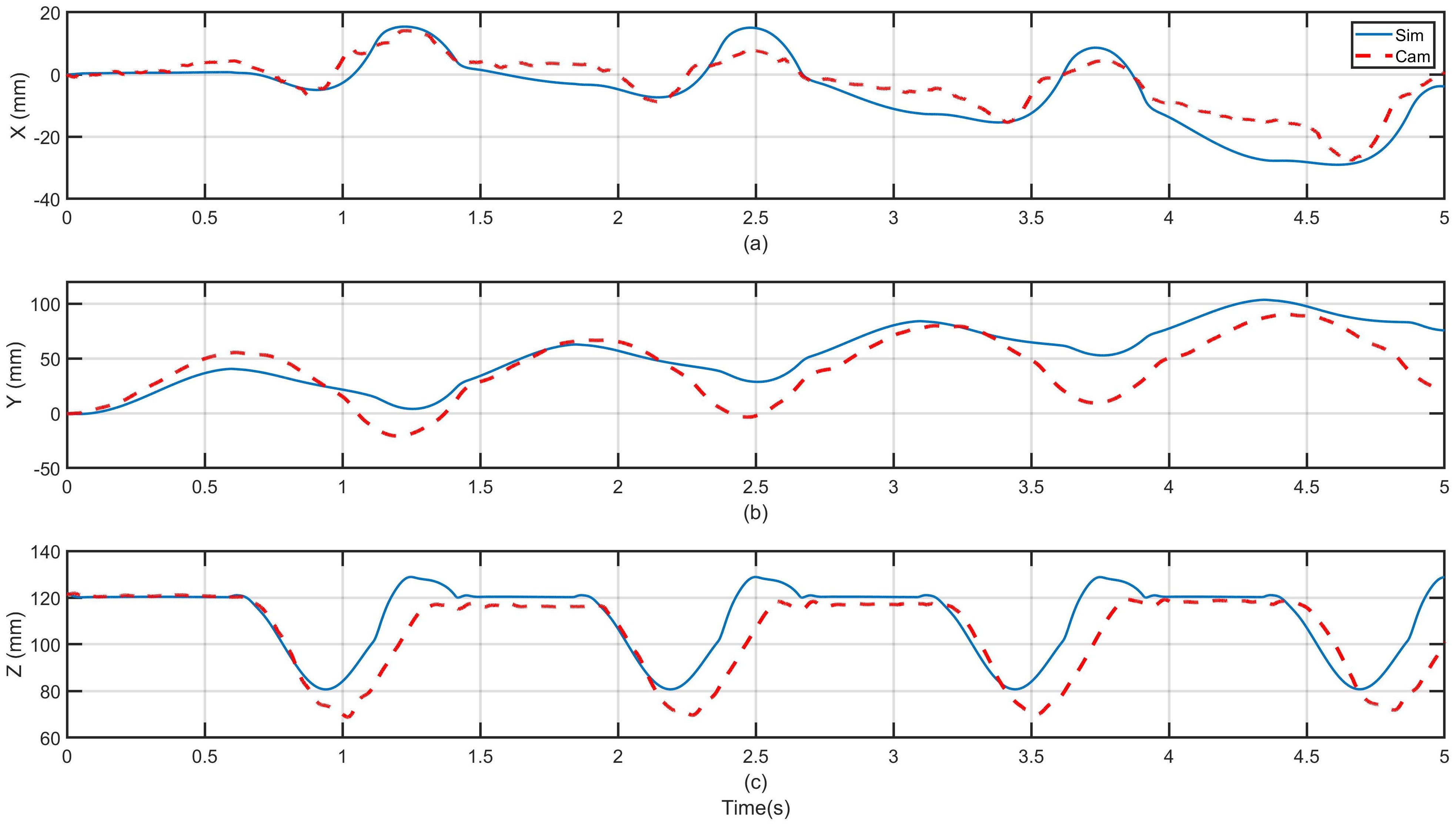}\\
  \caption{Main body COM position during the second damage scenario}
  \label{damage2_p}
\end{figure}
\begin{figure*}[htbp]
  \centering
  \includegraphics[width=.95\textwidth]{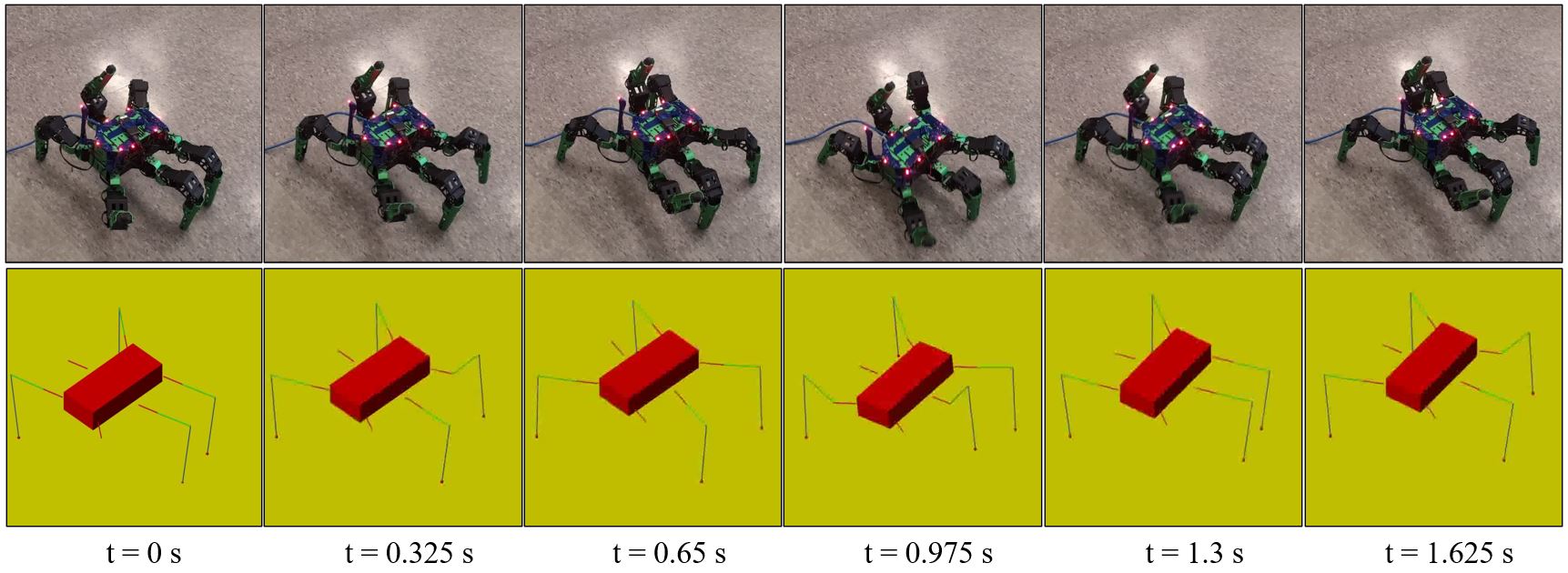}\\
  \caption{Snapshots of 1.625 seconds of model and robot motions motion for the second damage scenario}
  \label{damage2-snap}
\end{figure*}

\begin{figure}[hbt!]
  \centering
  \includegraphics[width=.45\textwidth]{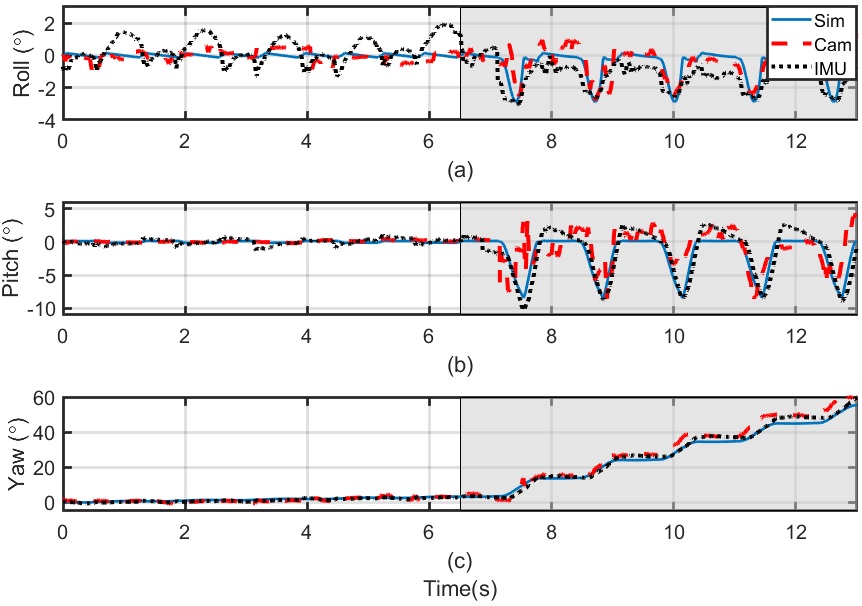}\\
  \caption{Main body orientation during the third damage scenario}
  \label{damage3_o}
\end{figure}
\begin{figure}[hbt!]
  \centering
  \includegraphics[width=.45\textwidth]{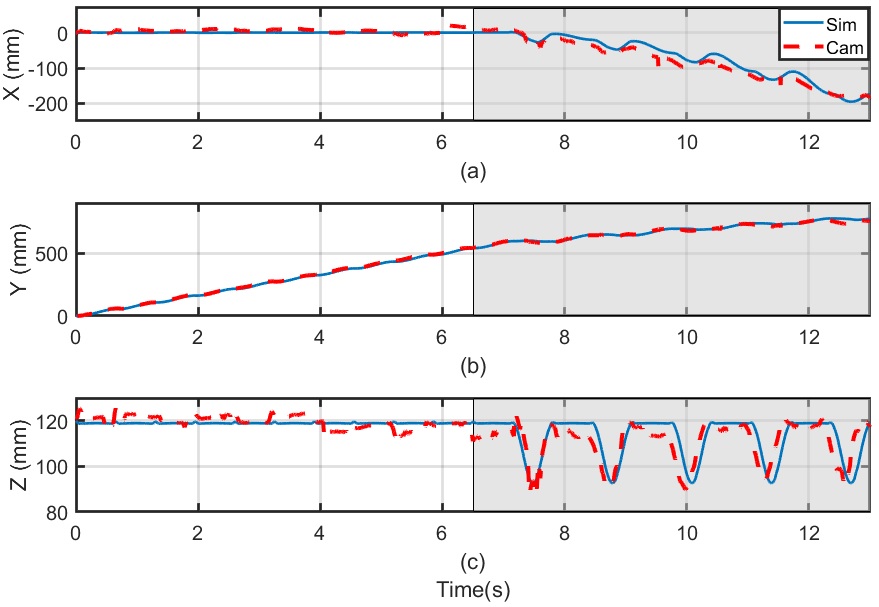}\\
  \caption{Main body COM position during the third damage scenario}
  \label{damage3_p}
\end{figure}

\section{Conclusion}
\label{conclusion}

Data-driven whole-body models, such as those based on neural networks, require retraining after damage, making real-time motion prediction infeasible until new data is incorporated. In addition, traditional analytical models rely on fixed physical parameters, making them rigid and difficult to adapt to unexpected morphological changes. This paper addresses these challenges by introducing an adaptable and modular analytical modeling approach for MLRs that autonomously adapts to damage scenarios without retraining or manual rederivation of equations.

Our contributions can be summarized as follows. First, we presented a singularity-free dynamic model using the Boltzmann-Hamel formulation, capturing all six DoFs of the main body and their influence on leg motions while surpassing traditional recursive models in computational efficiency. Second, we developed a modular version of this model that dynamically adjusts to morphological changes, eliminating the need for equation rederivation when mechanical damage occurs. Third, we introduced a computationally efficient formulation where the full system dynamics are generated by replicating equations derived for known leg morphologies, reducing runtime, complexity, and achieving real-time performance over three times faster than traditional methods. Finally, we validated our approach through extensive experiments on both healthy and damaged MLRs, leveraging both internal and external sensory data.

This work establishes a foundation for adaptive, high-speed modeling of MLRs, facilitating real-time control and damage recovery in unpredictable environments. In future work, we intend to extend our proposed model to include an autonomous self-modeling algorithm for MLRs and apply it for damage identification and recovery.

\bibliographystyle{IEEEtran}
\bibliography{bare_jrnl}

\appendices

\section{Proof of Lemma 3}
\label{app:mass}

\begin{proof}
Given $M_{ij}$ from  \eqref{Mij_new2} and knowing that the only variable part of the equation is in the form of $_i\Add^1_j$, to find $\frac {\partial M }{\partial \theta_{ik} }$, it suffices to find the derivative of these terms with respect to the joint angles. While we know $\frac {\partial \, _i\Add^1_j }{\partial \theta_{ik} } \, =\0_{6 \times 6 }$ for $k \geq j$, we calculate the partial derivatives of these terms for $k < j$,  :
\small
\begin{align} \nonumber
    &\frac {\partial \, _i\Add^1_j (.)}{\partial \theta_{ik} } = \frac {\partial }{\partial \theta_{ik} } \Ad^{-1}_{(e^{\widehat{\xi}_{i1} \theta_{i1}} \dotsb  e^{\widehat{\xi}_{ij} \theta_{ij}}) } (.)  \\ \nonumber
    &=\frac {\partial }{\partial \theta_{ik} } \Ad_{(e^{-\widehat{\xi}_{ij} \theta_{ij}}  \dotsb  e^{-\widehat{\xi}_{i1} \theta_{i1}}) } (.)  \\ \nonumber
    &=\frac {\partial }{\partial \theta_{ik} } \left( e^{-\widehat{\xi}_{ij} \theta_{ij}}  \dotsb  e^{-\widehat{\xi}_{i1} \theta_{i1}}  \, (.)^\wedge e^{\widehat{\xi}_{i1} \theta_{i1}}  \dotsb  e^{\widehat{\xi}_{ij} \theta_{ij}} \right) ^\vee  \\ \nonumber
    &=\frac {\partial }{\partial \theta_{ik} } \left( e^{-\widehat{\xi}_{ij} \theta_{ij}}  \dotsb  ( -\widehat{\xi}_{ik}  e^{-\widehat{\xi}_{ik} \theta_{ik}} )  \dotsb  e^{-\widehat{\xi}_{i1} \theta_{i1}}  (.)^\wedge e^{\widehat{\xi}_{i1} \theta_{i1}} \dotsb  e^{\widehat{\xi}_{ij} \theta_{ij}} \right.\\ \nonumber
    &\left. + e^{-\widehat{\xi}_{ij} \theta_{ij}}  \dotsb  e^{-\widehat{\xi}_{i1} \theta_{i1}} \, (.)^\wedge  e^{\widehat{\xi}_{i1} \theta_{i1}}  \dotsb ( e^{\widehat{\xi}_{ik} \theta_{ik}}  \widehat{\xi}_{ik}) \dotsb e^{\widehat{\xi}_{ij} \theta_{ij}} \right) ^\vee \\ \nonumber
    & =- \, _i\Add^{k-1}_j \left(  \widehat{\xi}_{ik} (\, _i\Add^{1}_k(.))^\wedge - (\, _i\Add^{1}_k(.))^\wedge \widehat{\xi}_{ik} \right)^\vee  \\ 
    &=- \, _i\Add^{k-1}_j [  \widehat{\xi}_{ik},(\, _i\Add^{1}_k(.))]^\wedge = -\,_i\Add^{k-1}_j \ad_{\widehat{\xi}_{ik}} \,_i\Add^{1}_k(.).
\end{align}
\normalsize
Similarly $\frac {\partial \, \xi^{'}_{i\beta} }{\partial \theta_{ik} } =\0_{6 \times 1 } $ for $k \geq \beta-1$; hence, for $k < \beta-1$ we have:
\small
\begin{align} \nonumber
    &\frac {\partial \xi^{'}_{i\beta} }{\partial \theta_{ik} } = \frac {\partial }{\partial \theta_{ik} } \Ad_{(e^{\widehat{\xi}_{i1} \theta_{i1}}  \dotsb  e^{\widehat{\xi}_{i\beta -1} \theta_{i\beta-1}}) } (.)  \\ \nonumber
    &= \frac {\partial }{\partial \theta_{ik} } \left( e^{\widehat{\xi}_{i1} \theta_{i1}} \dotsb e^{\widehat{\xi}_{i\beta-1} \theta_{i\beta-1}}  (.)^\wedge  e^{-\widehat{\xi}_{i\beta-1} \theta_{i\beta-1}}  \dotsb  e^{-\widehat{\xi}_{i1} \theta_{i1}} \right) ^\vee \\ \nonumber
    &= \frac {\partial }{\partial \theta_{ik} } \left( e^{\widehat{\xi}_{i1} \theta_{i1}} \dotsb ( e^{\widehat{\xi}_{ik} \theta_{ik}}  \widehat{\xi}_{ik} )  \dotsb e^{\widehat{\xi}_{i\beta-1} \theta_{i\beta-1}}  (.)^\wedge  e^{-\widehat{\xi}_{i\beta-1} \theta_{i\beta-1}} \right.\\ \nonumber 
    & \left. \dotsb  e^{-\widehat{\xi}_{i1} \theta_{i1}} +  e^{\widehat{\xi}_{i1} \theta_{i1}}  \dotsb  e^{\widehat{\xi}_{i\beta-1} \theta_{i\beta-1}} (.)^\wedge  e^{-\widehat{\xi}_{i\beta-1} \theta_{i\beta-1}} \right.\\ \nonumber 
    & \left.\dotsb  ( -\widehat{\xi}_{ik} e^{-\widehat{\xi}_{ik} \theta_{ik}} ) \dotsb  e^{-\widehat{\xi}_{i1} \theta_{i1}} \right) ^\vee \\ \nonumber 
    &= (\,_i\Add^{1}_k)^{-1} \left(  \widehat{\xi}_{ik} ( (\,_i\Add^{k+1}_{\beta-1})^{-1}(.))^\wedge - ( (\,_i\Add^{k+1}_{\beta-1})^{-1} (.))^\wedge \widehat{\xi}_{ik} \right)^\vee \\  
    &= (\,_i\Add^{1}_k)^{-1} [  \widehat{\xi}_{ik},( (\,_i\Add^{k+1}_{\beta-1})^{-1} (.))]^\wedge (\,_i\Add^{1}_k)^{-1} \ad_{\widehat{\xi}_{ik}} (\,_i\Add^{k+1}_{\beta-1})^{-1}(.).
\end{align}
\normalsize
The stated equations are the straightforward computations from these calculations.
\end{proof}

\label{prooftheorem}

\section{Proof of Proposition 2} \label{app:PRO2}

Firstly, We start with reformulating the first term of \eqref{chap4:C}:
\begin{align} \nonumber  
    & \sum\limits_{i=1}^{N}\sum\limits_{k=1}^{n_i}\frac {\partial M }{\partial \theta_{ik} } \,\dot\theta_{ik}  =  \sum\limits_{i=1}^{N}\sum\limits_{k=1}^{n_i}\frac {\partial (M_b + \sum\limits_{i=1}^{N} \bar\ex_i  M_{leg_i}) }{\partial \theta_{ik} } \,\dot\theta_{ik} = \\ 
    &\sum\limits_{i=1}^{N}\sum\limits_{k=1}^{n_i} \left( \frac {\partial M_b}{\partial \theta_{ik} } + \frac {\partial (\sum\limits_{i=1}^{N} \bar\ex_iM_{leg_i}) }{\partial \theta_{ik} } \right)\,\dot\theta_{ik} 
\end{align}
where $M_b$ is constant, leading to $\frac {\partial M_b }{\partial \theta_{ik} } = 0$, and:
\begin{align} \nonumber  
    &\sum\limits_{i=1}^{N}\sum\limits_{k=1}^{n_i} \frac {\partial ( \bar\ex_iM_{leg_i}) }{\partial \theta_{ik} } \,\dot\theta_{ik} = \sum\limits_{i=1}^{N}\sum\limits_{k=1}^{n_i} \frac {\partial ( \sum\limits_{j=1}^{n_i}\ex_{ij} M_{ij}) }{\partial \theta_{ik} } \,\dot\theta_{ik} \\ & \sum\limits_{i=1}^{N} \sum\limits_{j=1}^{n_i} \left( \sum\limits_{k=1}^{n_i} \ex_{ij} \frac {\partial M_{ij}(\theta_i) }{\partial \theta_{ik} } \right)\,\dot\theta_{ik}, \label{eq:aaa}
\end{align}
where $\frac {\partial M_{ij} }{\partial \theta_{ik} } = 0$ if $k>j$, therefore
\begin{align} \nonumber  
    & \sum\limits_{i=1}^{N}\sum\limits_{k=1}^{n_i}\frac {\partial M }{\partial \theta_{ik} } \,\dot\theta_{ik}  = \sum\limits_{i=1}^{N} \sum\limits_{j=1}^{n_i} \left( \sum\limits_{k=1}^{j} \ex_{ij} \frac {\partial M_{ij}(\theta_i) }{\partial \theta_{ik} } \right)\,\dot\theta_{ik} \,.
\end{align}

Secondly, we need to calculate $\mathcal{P} = Mv$ by implementing $M$ from \eqref{chap4:m_robot_NEW}:

\begin{gather} \nonumber
    \mathcal{P} \!\!=\!\! \begin{bmatrix}
    \mathcal{P}_v\\\mathcal{P}_\omega\\\mathcal{P}_\theta
\end{bmatrix} \!\!=\!\! Mv = \\ \nonumber \small\begin{bmatrix}
            I_b\!\! +\!\! \sum\limits_{i=1}^{N} \bar\ex_{i} M^{bb}_{leg_i}\!\!(\theta_i)  & \bar\ex_{1}\mathcal{M}^{b\theta}_{leg_1}\!\!(\theta_1)  & \cdots & \bar\ex_{N}\mathcal{M}^{b\theta}_{leg_N}\!\!(\theta_N)  \\
            \bar\ex_{1}\mathcal{M}^{\theta b}_{leg_1}\!\!(\theta_1)  &  \bar\ex_{1}\mathcal{M}^{\theta\theta}_{leg_1}\!\!(\theta_1) & \0 & \0 \\
            \vdots & \0 & \ddots & \0 \\
            \bar\ex_{N}\mathcal{M}^{\theta b}_{leg_N}\!\!(\theta_N)  & \0 & \0 & \bar\ex_{N}\mathcal{M}^{\theta\theta}_{leg_N}\!\!(\theta_N)  
        \end{bmatrix} \begin{bmatrix} V^b_{s,b} \\ \dot\theta \end{bmatrix} = \\\nonumber
        \begin{bmatrix}
        I_b V^b_{s,b} + \sum\limits_{i=1}^{N} \bar\ex_{i} M^{bb}_{leg_i} V^b_{s,b} +  \bar\ex_{i}\mathcal{M}^{\theta\theta}_{leg_i} \dot\theta_i ) \\ 
        \bar\ex_{1}\mathcal{M}^{\theta b}_{leg_1} V^b_{s,b} +  \bar\ex_{1}\mathcal{M}^{\theta\theta}_{leg_1} \dot\theta_1  \\ \vdots \\   \bar\ex_{N}\mathcal{M}^{\theta b}_{leg_N} V^b_{s,b} +  \bar\ex_{N}\mathcal{M}^{\theta\theta}_{leg_N} \dot\theta_N  \end{bmatrix} = \\
        \begin{bmatrix}
        I_b V^b_{s,b} + \sum\limits_{i=1}^{N} \bar\ex_i\sum\limits_{j=1}^{n_i} \ex_{ij}( M_{ij}^{bb} V^b_{s,b} +  \ \mathcal{M}_{ij}^{b\theta} \dot\theta_i ) \\ \begin{matrix} \bar\ex_1\sum\limits_{j=1}^{n_1} \ex_{1j} \bigg( \mathcal{M}^{\theta b}_{1j} V^b_{s,b} + \mathcal{M}^{\theta \theta}_{1j} \dot\theta_1 \bigg) \\ \vdots \\ \bar\ex_N\sum\limits_{j=1}^{n_N} \ex_{Nj} \bigg( \mathcal{M}^{b\theta}_{Nj} V^b_{s,b} +  \mathcal{M}^{\theta \theta}_{Nj} \dot\theta_N \bigg) \end{matrix} \end{bmatrix}. \label{eq:P_leg}
\end{gather}
The stated equations are the straightforward computations from these calculations.

\end{document}